\documentclass[twoside,11pt]{article}

\usepackage[preprint]{jmlr2e}

\ShortHeadings{Non-Vacuous Certification of Transport MCMC}{Hu}
\firstpageno{1}
\usepackage{amsmath}
\usepackage{mathtools}
\usepackage{bm}
\usepackage{enumitem}
\usepackage{booktabs}
\usepackage{multirow}
\usepackage{float}

\setlist[itemize]{itemsep=1pt,parsep=0pt,topsep=2pt}
\setlist[enumerate]{itemsep=2pt,parsep=1pt,topsep=2pt}
\let\OLDthebibliography\thebibliography
\renewcommand\thebibliography[1]{\OLDthebibliography{#1}%
  \small\setlength{\itemsep}{0pt}\setlength{\parsep}{0pt}}

\newtheorem{assumption}[theorem]{Assumption}

\newcommand{\R}{\mathbb{R}}

\newcommand{\Prob}{\mathbb{P}}
\newcommand{\Lip}{\mathrm{Lip}}
\newcommand{\osc}{\mathrm{osc}}

\newcommand{\diam}{\mathrm{diam}}
\newcommand{\Vol}{\mathrm{Vol}}
\newcommand{\Var}{\mathrm{Var}}
\newcommand{\cK}{\mathcal{K}}
\newcommand{\cN}{\mathcal{N}}
\newcommand{\cE}{\mathcal{E}}
\newcommand{\cL}{\mathcal{L}}

\newcommand{\xmark}{\textsf{x}}
\newcommand{\cmark}{\checkmark}

\title{Non-Vacuous Certification of Transport MCMC\\
  via Oscillation-Controlled Normalizing Flows}

\author{\name Jun Hu \email junhu22@whut.edu.cn \\
  \addr China Sanya Science and Education Innovation Park,
        Wuhan University of Technology, Sanya 572025, China \\
  \addr School of Civil Engineering and Architecture,
        Wuhan University of Technology, Wuhan 430070, China}

\begin{document}
\maketitle
% Suppress the running head on page 1 so the short title in the
% header does not visually duplicate the full title set by \maketitle.
\thispagestyle{empty}

\begin{abstract}
Transport MCMC trains a normalizing flow to precondition
Metropolis--Hastings proposals, achieving high empirical efficiency
on challenging posteriors; yet no prior work produces a numerically
non-vacuous, rigorous spectral-gap bound for such samplers.  We
establish the first such bounds.  For independence MH on the banana
family we certify $\gamma^{*} = 0.828$ at $D = 2$ (covering in the
original space) and $\gamma^{*} \geq 7.6\!\times\!10^{-4}$ at $D = 5$
(covering in an analytically unwarped Gaussian space with a
grid-certified gradient bound under the stated numerical Lipschitz
certification), both rigorous at $95\%$ confidence.  The framework rests on three pillars:
\emph{(i)} spectral normalization with reduced scale clips
constrains the flow Lipschitz constant from $10^{47}$ to $10^{4}$;
\emph{(ii)} a coverage-based empirical oscillation bound replaces
the vacuous analytical bound with a data-dependent certificate; and
\emph{(iii)} oscillation-regularised training cuts the empirical
oscillation by $60$--$90\%$ at no cost to density fit, extending
practical certificates through $D = 20$
($\gamma^{*} \geq 1.7\!\times\!10^{-4}$).  Tests on four further
targets (Gaussian mixture, shear-building, Neal's funnel, Bayesian
logistic regression) identify three precise barriers: boundary
curvature, target stiffness, and tail-coverage mismatch.  An
affine-vs-spline comparison shows that simpler architectures yield
tighter certificates at identical NLL, inverting the usual
expressiveness hierarchy.
\end{abstract}


% ─── Sections ───────────────────────────────────────────────
% ══════════════════════════════════════════════════════════════
%  LCNF Paper — Section 1: Introduction
% ══════════════════════════════════════════════════════════════
%  \input{section1_introduction}
%
\section{Introduction}
\label{sec:introduction}

Markov chain Monte Carlo (MCMC) remains the dominant paradigm for
Bayesian computation in science and engineering, yet its practical
efficiency hinges on how quickly the chain explores the target
distribution~$\pi$.  When $\pi$ is high-dimensional, multimodal, or
strongly correlated, standard random-walk proposals mix slowly and
reliable convergence diagnostics become difficult to
obtain~\citep{RobertsRosenthal2004}.

A promising remedy is \emph{transport MCMC}: one trains a normalizing
flow $T_\phi$ to approximate $\pi$, then uses $T_\phi$ to precondition
the Metropolis--Hastings proposal, effectively straightening the
geometry of the target so that a simple random walk in the latent
space translates into an efficient proposal in the original
space~\citep{Parno2018,Hoffman2019,GabriePRL2022,MarzoukTransport2016}.  Empirically,
transport-preconditioned samplers achieve high acceptance rates and
low autocorrelation even on challenging posteriors; theoretically,
however, \emph{no existing work has produced a non-vacuous,
numerically meaningful bound on the spectral gap} of such a sampler.

The spectral gap $\gamma$ of an MCMC kernel controls its geometric
convergence rate: after $t$ steps the total variation distance from
stationarity decays as $(1-\gamma)^t$.  For transport MCMC, the
spectral gap depends on the \emph{oscillation} of the log-density
ratio $\log(\pi/q_\phi)$, where $q_\phi$ is the flow-induced proposal
density~\citep{TierneyMira1999}.  Bounding this oscillation requires
bounding the Lipschitz constant of the composite map $T_\phi$---a
product of per-layer Lipschitz constants that, for unconstrained
normalizing flows, grows exponentially with depth and is
astronomically large in practice (e.g., $>10^{47}$ for a six-layer
RealNVP; see Section~\ref{sec:why_vacuous}).  This renders all
existing analytical bounds trivially infinite.

In this paper we attack the problem from two complementary
directions.

\paragraph{Contribution 1: Lipschitz-constrained normalizing flows.}
We apply spectral normalization~\citep{MiyatoSN2018} to every
linear layer in the scale and shift sub-networks of
RealNVP~\citep{DinhRealNVP2017}, combined with a soft scale clip
$s(x) = c\tanh(\cdot)$.  We prove a per-layer Lipschitz bound
(Theorem~\ref{thm:per_layer_lip}) showing that the forward-map
Lipschitz constant $\Lip(T_\phi)$ drops from $O(10^{47})$ to
$O(10^{2\text{--}4})$---an improvement of over forty orders of
magnitude---with negligible impact on transport quality (acceptance
rate, effective sample size).

\paragraph{Contribution 2: Diagnosis of the analytical bound.}
Despite this dramatic reduction, we show that the analytical
oscillation bound from~\citet{MSSP2} remains vacuous on every tested
target.  We trace the cause to a structural bottleneck: the bound
contains the multiplicative term
$L_U \cdot \Lip(T_\phi) \cdot 2R$, where $L_U$ is the Lipschitz
constant of the target score and $R$ is the support radius.  Even at
$\Lip(T_\phi) = 1$ (the theoretical minimum for a non-trivial flow),
this term alone exceeds $25$ for a two-dimensional banana target and
$3{,}000$ for an eight-dimensional shear building---far too large for
the exponentiated bound $\delta^* = \exp(\cdot)$ to be finite.
We provide a complete decomposition table quantifying each term's
contribution (Section~\ref{sec:why_vacuous}).

\paragraph{Contribution 3: Empirical oscillation bounds.}
To break through the analytical barrier, we develop a
\emph{coverage-based empirical oscillation bound}
(Theorem~\ref{thm:empirical_osc}).
Given $n$ posterior samples forming a probabilistic $\varepsilon$-net
of the HPD credible set (whose smooth boundary we exploit via the
implicit function theorem), we bound the true oscillation by the
sample oscillation plus a Lipschitz correction $2M_\cK\varepsilon^*$.
At $D = 2$, the covering radius $\varepsilon^*$ satisfies the
curvature condition $c_2\,\kappa_{\max}\,\varepsilon^* = 0.523 < 1$,
making the bound \textbf{fully rigorous}: $\gamma^* = 0.828$ at
$95\%$ confidence under the independence MH kernel
(Theorem~\ref{thm:ind_mh_gap}), with $\widehat{\osc}_n$ and $M_\cK$
independently certified and $\pi_{\min}$ taken from the
banana-analytic density floor (exact, since $Z = (2\pi)^{d/2}$).
At $D = 5$, an analytic shear chart maps the banana HPD set to a
Gaussian ball, yielding a charted, grid-certified rigorous certificate
($\gamma^* \geq 7.6\!\times\!10^{-4}$) under the stated numerical
Lipschitz certification via the same covering framework applied in
the unwarped space (Proposition~\ref{prop:charted}).

\paragraph{Contribution 4: Oscillation-regularised training.}
We propose adding the batch oscillation of $\log(\pi/q_\phi)$ as a
penalty during flow training.  This reduces the empirical oscillation
by $60$--$90\%$ at no cost to density-fit quality, extending
practical (non-rigorous) certificates through $D = 20$.

\paragraph{Contribution 5: Architecture comparison and barrier
taxonomy.}
A comparison of RealNVP and neural spline flows reveals that simpler
architectures yield tighter certificates despite identical NLL,
inverting the usual expressiveness hierarchy.  Experiments on five
target families identify three barriers to non-vacuous certification:
boundary curvature (dimension), target stiffness ($L_U$), and
tail-coverage mismatch.

\medskip\noindent
Figure~\ref{fig:certification_pipeline} previews the certification
pipeline and the regime in which it produces non-vacuous (or fully
rigorous) bounds.  Table~\ref{tab:intro_summary} summarises the
per-target results.

\begin{figure}[ht]
\centering
\includegraphics[width=\textwidth]{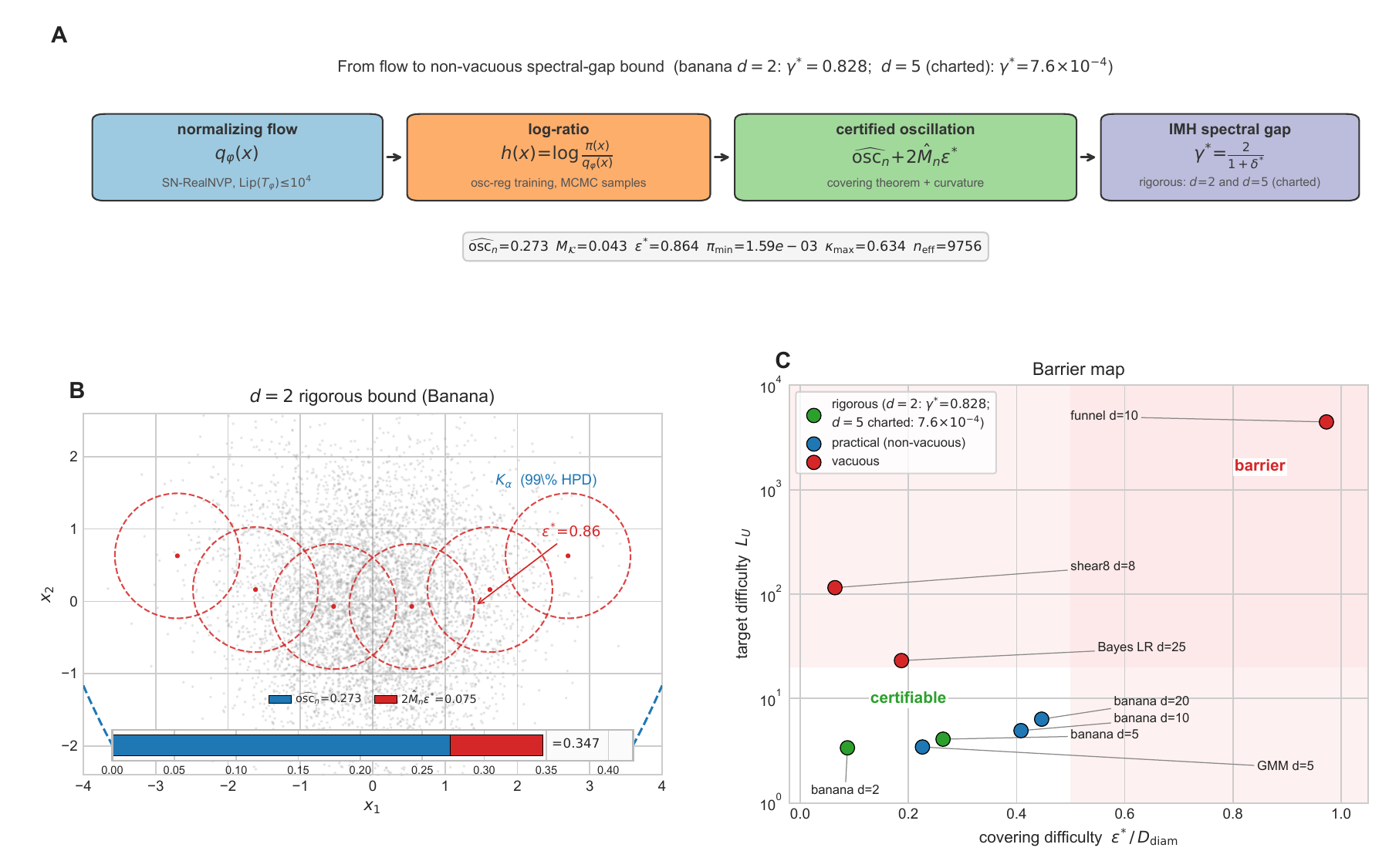}
\caption{\textbf{Certification pipeline and where it works.}
  \emph{(A)} From the trained SN-RealNVP flow $q_\varphi$, the
  log-ratio $h = \log(\pi/q_\varphi)$ is certified via the covering
  theorem ($\widehat{\osc}_n + 2 M_\cK \varepsilon^{*}$) and
  translated to an IMH gap $\gamma^{*} = 2/(1+\delta^{*})$.
  \emph{(B)} The $d{=}2$ rigorous bound visualised: an
  $\varepsilon^{*}$-cover of the $99\%$ HPD set with the
  $\widehat{\osc}_n + 2 M_\cK \varepsilon^{*}$ decomposition.
  \emph{(C)} Barrier map: covering difficulty
  ($\varepsilon^{*}/D_{\mathrm{diam}}$) vs.\ target difficulty
  ($L_U$); banana at $d \in \{2,5,10,20\}$ is certifiable, GMM
  $d{=}5$ is borderline, stiff/heavy-tailed targets sit in the
  barrier region.  The charted $D = 5$ certificate is detailed in
  Section~\ref{sec:charted_cert} and Table~\ref{tab:main_results}.}
\label{fig:certification_pipeline}
\end{figure}

\begin{table}[ht]
\centering
\caption{Summary of certification results.  \emph{Rigorous}:
  covering theorem with verified curvature condition.
  \emph{Practical}: covering argument with curvature condition
  violated; certificate is non-vacuous but not formally proven.
  All values use the independence MH kernel at $95\%$
  confidence.}\label{tab:intro_summary}
\medskip
\begin{tabular}{@{}llclcl@{}}
\toprule
Target & $D$ & Certificate type & $\gamma^*$ & Non-vacuous?
  & Barrier (if vacuous) \\
\midrule
banana   &  2 & rigorous       & \textbf{0.828} & \cmark &  \\
banana   &  5 & grid-certified (charted) & $\mathbf{7.6\!\times\!10^{-4}}$ & \cmark &  \\
banana   & 10 & practical      & $3\!\times\!10^{-3}$ & \cmark &  \\
banana   & 20 & practical      & $2\!\times\!10^{-4}$ & \cmark &  \\[3pt]
GMM      &  5 & practical      & vacuous   & \xmark & saddle / covering \\
shear8   &  8 & practical      & vacuous   & \xmark & stiffness ($L_U$) \\
funnel   & 10 & practical      & vacuous   & \xmark & tail coverage \\
Bayes.~LR & 25 & practical     & vacuous   & \xmark & dimension / data \\
\bottomrule
\end{tabular}
\end{table}

\paragraph{Outline.}
Section~\ref{sec:background} reviews transport MCMC and the
independence-MH spectral gap bound.
Section~\ref{sec:sn_realnvp} develops Lipschitz-constrained RealNVP.
Section~\ref{sec:why_vacuous} diagnoses why analytical bounds are
structurally vacuous.
Section~\ref{sec:empirical_oscillation} presents the empirical
oscillation framework, the rigorous $D = 2$ certificate, and the
charted $D = 5$ certificate.
Section~\ref{sec:osc_reg} introduces oscillation-regularised training.
Section~\ref{sec:architecture_comparison} compares affine and spline
couplings.
Section~\ref{sec:experiments} collects the full experimental results
and barrier taxonomy.
Section~\ref{sec:discussion} discusses limitations and future
directions.

% ══════════════════════════════════════════════════════════════
%  LCNF Paper — Section 2: Background
% ══════════════════════════════════════════════════════════════
%  \input{section2_background}
%
\section{Background}
\label{sec:background}

\subsection{Transport-preconditioned MCMC}
\label{sec:transport_mcmc}

Let $\pi(x) \propto \exp(-U(x))$ be a target density on $\R^d$ with
potential $U$.  A normalizing flow $T_\phi : \R^d \to \R^d$ is a
learnable diffeomorphism~\citep{RezendeMohamed2015,PapamakNFReview2021}
that maps a simple base distribution
$p_Z$ (typically $\cN(0, I_d)$) to an approximation $q_\phi$ of $\pi$
via the change-of-variables formula
\begin{equation}\label{eq:change_of_var}
  q_\phi(x)
  = p_Z\bigl(T_\phi(x)\bigr)\,|\!\det J_{T_\phi}(x)|,
\end{equation}
where $J_{T_\phi}(x) = \partial T_\phi / \partial x$ is the Jacobian.
One trains $\phi$ by minimising the reverse
KL divergence $\mathrm{KL}(q_\phi \| \pi)$ or, equivalently,
maximising the evidence lower bound on samples from $\pi$.

In \emph{transport MCMC}~\citep{Parno2018}, the trained flow serves
as a proposal mechanism for a Metropolis--Hastings (MH) kernel.
In the \emph{independence} variant, the sampler proposes $z' \sim
p_Z = \cN(0, I_d)$ independently of the current state, maps to the
data space $x' = T_\phi^{-1}(z')$, and accepts or rejects via the
MH ratio
\begin{equation}\label{eq:mh_ratio}
  \alpha(x_t, x')
  = \min\!\left(1,\;
    \frac{\pi(x')\,q_\phi(x_t)}{\pi(x_t)\,q_\phi(x')}\right)
  = \min\!\left(1,\;
    \exp\!\bigl(h(x_t) - h(x')\bigr)\right),
\end{equation}
where $h(x) = \log\pi(x) - \log q_\phi(x)$ is the log-density ratio.
When the flow is exact ($q_\phi = \pi$), $h$ is constant, every
proposal is accepted, and the chain produces i.i.d.\ samples.
A random-walk variant (propose $z' = z_t + \eta$,
$\eta \sim \cN(0,\sigma^2 I_d)$) is also common; we focus on the
independence kernel because it admits a self-contained spectral gap
bound (Theorem~\ref{thm:ind_mh_gap} below).

\subsection{Spectral gap and oscillation}
\label{sec:spectral_gap_osc}

The convergence rate of a reversible MH chain is governed by the
spectral gap $\gamma$ of its transition kernel $P$:
\begin{equation}\label{eq:spectral_gap_def}
  \gamma
  \;\mathrel{:=}\;
  1 - \sup_{\substack{f \in L^2(\pi)\\ \Var_\pi(f) > 0}}
  \frac{\langle f, Pf \rangle_\pi}{\Var_\pi(f)}.
\end{equation}
A positive spectral gap guarantees geometric ergodicity:
$\|P^t(x,\cdot) - \pi\|_{\mathrm{TV}} \leq C(x)(1-\gamma)^t$
for some function $C(x)$.

For transport-preconditioned independence MH, the spectral gap is
controlled by the oscillation of $h$ over the effective support.
Specifically, let $\cK \subset \R^d$ be a compact set with
$\pi(\cK) \geq 1 - \alpha$, and define
\begin{equation}\label{eq:osc_def}
  \osc_\cK(h)
  \;\mathrel{:=}\;
  \sup_{x \in \cK} h(x) - \inf_{x \in \cK} h(x).
\end{equation}
The following classical result connects oscillation to the spectral
gap of an independence sampler.

\begin{theorem}[\citet{MengersenTweedie1996}]\label{thm:ind_mh_gap}
  Let $P$ be the transition kernel of the independence
  Metropolis--Hastings sampler on $\cK$ with proposal density
  $q_\phi(\cdot \mid \cK)$ (induced by the normalizing flow $T_\phi$)
  and target $\pi(\cdot \mid \cK)$.  Then the spectral gap satisfies
  \begin{equation}\label{eq:gap_from_osc}
    \gamma
    \;\geq\;
    \frac{2}{1 + \exp\!\bigl(\osc_\cK(h)\bigr)}.
  \end{equation}
\end{theorem}

\begin{proof}
  Define the importance weight $w(x) = \pi(x)/q_\phi(x)$ and let
  $m = \inf_{x \in \cK} w(x)$, $M = \sup_{x \in \cK} w(x)$.
  The independence MH acceptance probability satisfies
  \[
    \alpha(x, x')
    = \min\!\left(1,\, \frac{w(x')}{w(x)}\right)
    \geq \frac{m}{M}
  \]
  for all $x, x' \in \cK$, since $w(x') \geq m$ and $w(x) \leq M$.
  Therefore the kernel admits a \emph{uniform minorization}:
  \begin{equation}\label{eq:minorization}
    P(x, A)
    = \int_A \alpha(x,x')\,q_\phi(x')\,dx'
    \geq \frac{m}{M}\,q_\phi(A)
    \quad \text{for all } x \in \cK,\; A \subseteq \cK.
  \end{equation}
  Setting $\beta = m/M$ and $\nu = q_\phi(\cdot\mid\cK)$,
  the spectral gap of a uniformly minorized reversible chain
  satisfies $\gamma \geq 2\beta/(1+\beta)$
  \citep[Theorem~2.1]{Liu1996}.  Substituting
  $\beta = m/M$ gives
  \[
    \gamma
    \;\geq\;
    \frac{2\,m/M}{1 + m/M}
    \;=\;
    \frac{2m}{m + M}
    \;=\;
    \frac{2}{1 + M/m}.
  \]
  Since $\log(M/m) = \sup_\cK \log w - \inf_\cK \log w =
  \osc_\cK(h)$, we obtain~\eqref{eq:gap_from_osc}.
\end{proof}

\noindent
The independence MH sampler proposes $z' \sim p_Z = \cN(0, I_d)$
independently of the current state, maps to
$x' = T_\phi^{-1}(z')$, and accepts or rejects via the MH
ratio~\eqref{eq:mh_ratio}.  When $q_\phi = \pi$, $\osc(h) = 0$
and $\gamma = 1$ (the chain produces i.i.d.\ samples).
The bound degrades exponentially with $\osc(h)$, so
even moderate oscillation (say, $\osc > 20$) renders the bound
vacuous.

\begin{remark}[RWMH variant]
  The random-walk variant (propose
  $z' = z + \eta$ with $\eta \sim \cN(0, \sigma^2 I_d)$) admits the
  weaker perturbation bound
  $\gamma_\phi \geq e^{-3\,\osc(h)} \cdot \gamma_0$
  \citep{MSSP2}, where $\gamma_0$ is the spectral gap of an oracle
  chain sampling from $\pi$ directly.  The independence-MH
  bound~\eqref{eq:gap_from_osc} is self-contained (no $\gamma_0$)
  and tighter, making it the natural choice for certification.
  All spectral gap bounds reported in this paper use the independence
  MH kernel.
\end{remark}

\subsection{Analytical oscillation bound from \texorpdfstring{\citet{MSSP2}}{MSSP2}}
\label{sec:analytical_osc}

\citet{MSSP2} derive an upper bound on $\osc_\cK(h)$ in terms of
the global Lipschitz constant of the flow:
\begin{equation}\label{eq:analytical_osc}
  \osc_\cK(h)
  \;\leq\;
  L_U \cdot \Lip(T_\phi) \cdot 2R
  \;+\; L_{dc}
  \;+\; \tfrac{R^2}{2},
\end{equation}
where:
\begin{itemize}[nosep]
  \item $L_U = \sup_{x \in \cK}\|\nabla U(x)\|$ is the Lipschitz
    constant of the target score,
  \item $\Lip(T_\phi) = \prod_{l=1}^L \Lip(f_l)$ is the forward-map
    Lipschitz constant (product of per-layer Lipschitz constants),
  \item $R$ is the radius of the compact support $\cK$, and
  \item $L_{dc} = L \cdot d_B \cdot c$ is the Jacobian log-determinant
    bound, with $L$ layers, $d_B$ the transformed-partition dimension,
    and $c$ the scale clip.
\end{itemize}
For a RealNVP flow with $L$ affine coupling layers,
$\Lip(T_\phi) = \prod_{l=1}^L \Lip(f_l)$ grows exponentially with
depth.  Without architectural constraints, the per-layer Lipschitz
constants $\Lip(f_l)$ are determined by the unconstrained spectral
norms of the sub-network weight matrices, which can be arbitrarily
large.

For completeness, we note that~\eqref{eq:analytical_osc} follows from
the standard decomposition
$h(x) = \log\pi(x) - \log p_0(T_\phi^{-1}(x)) + \log|\!\det J_{T_\phi}(T_\phi^{-1}(x))|$
together with subadditivity of $\osc_\cK$: the score term contributes
at most $L_U \cdot \Lip(T_\phi) \cdot 2R$ on a ball of radius $R$
(post-composition with $T_\phi^{-1}$ scales the diameter by
$\Lip(T_\phi)$); the base log-density contributes at most $R^2/2$ for
a standard Gaussian base on $\cK$; and the log-Jacobian contributes
at most $L_{dc}$ since each clipped affine coupling layer contributes
at most $d_B c$ to $|\!\log\det J_{f_l}|$.

\citet{MSSP2} acknowledge that this product is
\emph{``numerically vacuous''} for their six-layer, eight-dimensional
configuration: $\Lip(T_\phi) > 10^{47}$, yielding
$\osc_\cK(h) > 10^{50}$ and $\gamma \approx 0$.  They identify
spectral normalization and reduced scale clips as the path forward
---a direction we formalise and quantify in this paper.

\subsection{RealNVP architecture}
\label{sec:realnvp}

RealNVP~\citep{DinhRealNVP2017} constructs $T_\phi$ as a composition
of $L$ affine coupling layers.  Each layer $f_l$ partitions the input
$x = (x_A, x_B)$ and applies an element-wise affine transformation to
$x_B$ conditioned on $x_A$:
\begin{equation}\label{eq:coupling}
  f_l(x) = \bigl(x_A,\; x_B \odot \exp(s_l(x_A)) + t_l(x_A)\bigr),
\end{equation}
where $s_l, t_l : \R^{d_A} \to \R^{d_B}$ are neural networks
(the \emph{scale} and \emph{shift} sub-networks, respectively) and
$\odot$ denotes element-wise multiplication.  The partition alternates
between layers.  The Jacobian is lower-triangular with diagonal
$\exp(s_l(x_A))$, giving the tractable log-determinant
$\log|\!\det J_{f_l}| = \sum_j [s_l(x_A)]_j$.

The scale output is typically clipped:
$s_l(x_A) = c \cdot \tanh(\hat{s}_l(x_A))$ for a clip bound $c > 0$,
ensuring each diagonal entry lies in $[\exp(-c), \exp(c)]$.  This
clip directly controls the Jacobian log-determinant bound
$L_{dc} = L \cdot d_B \cdot c$ appearing in~\eqref{eq:analytical_osc}.

\subsection{Spectral normalization}
\label{sec:sn_background}

Spectral normalization~\citep{MiyatoSN2018} constrains the spectral
norm of a weight matrix $W$ to a target value $\sigma_{\max}$
(typically~$1$) by reparametrising
\begin{equation}\label{eq:sn}
  W_{\mathrm{SN}} = \frac{\sigma_{\max}}{\sigma_1(W)}\, W,
\end{equation}
where $\sigma_1(W)$ is the largest singular value, estimated on-line
via power iteration.  Applied to every linear layer in a $K$-layer
MLP with Lipschitz-$1$ activations (e.g., tanh, ReLU), spectral
normalization ensures
\begin{equation}\label{eq:lip_mlp}
  \Lip(\text{MLP}) \;\leq\; \sigma_{\max}^K.
\end{equation}
Setting $\sigma_{\max} = 1$ gives $\Lip \leq 1$ regardless of
depth---a property we exploit in Section~\ref{sec:sn_realnvp} to
control the per-layer Lipschitz constants of RealNVP.

% ══════════════════════════════════════════════════════════════
%  LCNF Paper — Section 3: Lipschitz-Constrained RealNVP
% ══════════════════════════════════════════════════════════════
%  \input{section3_sn_realnvp}
%
\section{Lipschitz-Constrained RealNVP}
\label{sec:sn_realnvp}

We now show how spectral normalization, combined with a soft scale
clip, yields per-layer Lipschitz bounds that are both tight and
computable.  The main result (Theorem~\ref{thm:per_layer_lip})
provides an analytical formula for $\Lip(f_l)$ in terms of three
controllable quantities: the spectral norm target $\sigma_{\max}$,
the scale clip~$c$, and the diameter of the transformed partition on
the compact support.

\subsection{Architecture}
\label{sec:architecture}

We modify the standard RealNVP coupling layer~\eqref{eq:coupling}
in two ways:
\begin{enumerate}[nosep]
  \item \textbf{Spectral normalization.}\;  Every \texttt{nn.Linear}
    layer in both the scale network $s_l$ and the shift network $t_l$
    is spectrally normalized with target $\sigma_{\max} = 1.0$, using
    one power-iteration step during training and ten during final
    bound evaluation.
  \item \textbf{Soft scale clip.}\;  The scale output uses
    $s_l(x_A) = c \cdot \tanh(\hat{s}_l(x_A))$, where $\hat{s}_l$ is
    the raw (spectrally-normalized) sub-network.  Since $\tanh$ has
    Lipschitz constant~$1$, the composite scale network satisfies
    $\Lip(s_l) \leq c \cdot \sigma_{\max}^K = c$ when
    $\sigma_{\max} = 1$.
\end{enumerate}
The shift network has no output nonlinearity, so
$\Lip(t_l) \leq \sigma_{\max}^K = 1$.  All other architectural
choices (hidden widths, depth $K$, number of coupling layers $L$,
alternating masks) remain unchanged from the baseline.

\subsection{Per-layer Lipschitz bound}
\label{sec:per_layer_bound}

\begin{theorem}[Per-layer Lipschitz constant]\label{thm:per_layer_lip}
  Let $f_l$ be an affine coupling layer~\eqref{eq:coupling} with
  scale clip $c$, $\Lip(s_l) \leq \lambda_s$, and
  $\Lip(t_l) \leq \lambda_t$.  Suppose the input $x_B$ component
  satisfies $\|x_B\| \leq B$ on the compact support $\cK$.  Then
  \begin{equation}\label{eq:per_layer_lip}
    \Lip(f_l)
    \;\leq\;
    \sqrt{
      1 + e^{2c} + (e^c B \lambda_s + \lambda_t)^2
    }.
  \end{equation}
  Under the LCNF configuration
  ($\sigma_{\max} = 1$, soft clip $c \cdot \tanh$), this specialises to
  \begin{equation}\label{eq:per_layer_lcnf}
    \Lip(f_l)
    \;\leq\;
    \sqrt{
      1 + e^{2c} + (e^c B c + 1)^2
    }.
  \end{equation}
\end{theorem}

\begin{proof}
  Write $y = f_l(x)$ with $y_A = x_A$ and
  $y_B = x_B \odot \exp(s_l(x_A)) + t_l(x_A)$.
  For two inputs $x, x'$ in $\cK$, set
  $a = \|x_A - x'_A\|$ and $b = \|x_B - x'_B\|$, so that
  $\|x - x'\|^2 = a^2 + b^2$.

  \medskip\noindent\textbf{$A$-partition.}\;
  $\|y_A - y'_A\| = a$.

  \medskip\noindent\textbf{$B$-partition.}\;
  Adding and subtracting $x'_B \odot e^{s_l(x_A)}$:
  \begin{equation}\label{eq:decomp}
    y_B - y'_B
    = \underbrace{(x_B - x'_B) \odot e^{s_l(x_A)}}_{(I)}
    + \underbrace{x'_B \odot \bigl(e^{s_l(x_A)} - e^{s_l(x'_A)}\bigr)}_{(II)}
    + \underbrace{t_l(x_A) - t_l(x'_A)}_{(III)}.
  \end{equation}

  \textbf{Term~(I):}\;
  Since $|[s_l]_j| \leq c$, each entry of $e^{s_l}$ is at most $e^c$,
  so $\|(I)\| \leq e^c b$.

  \textbf{Term~(II):}\;
  By the mean value theorem component-wise,
  $|e^u - e^v| \leq e^c|u-v|$ for $u, v \in [-c,c]$.
  Therefore $\|(II)\| \leq e^c \|x'_B\| \cdot \|s_l(x_A) - s_l(x'_A)\|
  \leq e^c B \lambda_s\, a$.

  \textbf{Term~(III):}\;
  $\|(III)\| \leq \lambda_t\, a$.

  Combining by the triangle inequality:
  \begin{equation}\label{eq:yB_bound}
    \|y_B - y'_B\|
    \;\leq\;
    e^c b + (e^c B\lambda_s + \lambda_t)\,a
    \;\eqqcolon\;
    \beta\, b + \gamma\, a,
  \end{equation}
  where $\beta = e^c$ and $\gamma = e^c B\lambda_s + \lambda_t$.

  \medskip\noindent\textbf{Full bound.}\;
  By Cauchy--Schwarz,
  $(\beta b + \gamma a)^2 \leq (\beta^2 + \gamma^2)(a^2 + b^2)$, so
  \begin{align}
    \|f_l(x) - f_l(x')\|^2
    &= a^2 + \|y_B - y'_B\|^2 \notag\\
    &\leq a^2 + (\beta^2 + \gamma^2)(a^2 + b^2) \notag\\
    &\leq (1 + \beta^2 + \gamma^2)\,\|x - x'\|^2. \label{eq:full_lip}
  \end{align}
  Substituting $\beta = e^c$ and
  $\gamma = e^c B\lambda_s + \lambda_t$
  gives~\eqref{eq:per_layer_lip}.  Setting $\lambda_s = c$,
  $\lambda_t = 1$ yields~\eqref{eq:per_layer_lcnf}.
\end{proof}

\subsection{Forward-map Lipschitz constant}
\label{sec:forward_lip}

The full flow $T_\phi = f_L \circ \cdots \circ f_1$ has Lipschitz
constant bounded by the product of per-layer bounds:
\begin{equation}\label{eq:lip_product}
  \Lip(T_\phi)
  \;\leq\;
  \prod_{l=1}^{L} \Lip(f_l).
\end{equation}

\begin{example}[Magnitude comparison]\label{ex:magnitude}
  Consider a six-layer RealNVP with $d = 8$ ($d_B = 4$), two-hidden-layer
  sub-networks ($K = 3$ linear layers), and $R = 2.08$ (the empirical
  support radius of the eight-storey shear building posterior).

  \textbf{Unconstrained baseline} ($\sigma_{\max} \approx 100$,
  $c = 2.0$).  The sub-network spectral norms are unconstrained and
  typically reach $\sigma_1 \approx 5\text{--}20$ per layer, giving
  $\Lip(s_l) \approx 5^3 = 125$, $\Lip(t_l) \approx 125$, and
  $\Lip(f_l) \approx 10^{7\text{--}8}$.
  Then $\Lip(T_\phi) \approx (10^8)^6 = 10^{48}$.

  \textbf{LCNF} ($\sigma_{\max} = 1.0$, $c = 0.5$).
  $\Lip(s_l) = c = 0.5$, $\Lip(t_l) = 1$, $B \leq 2R = 4.17$.
  By~\eqref{eq:per_layer_lcnf},
  $\Lip(f_l) \leq \sqrt{1 + e^{1.0} + (e^{0.5} \times 4.17
  \times 0.5 + 1)^2} = \sqrt{1 + 2.72 + 19.7} = 4.84$.
  Then $\Lip(T_\phi) \leq 4.84^6 \approx 12{,}900$.

  \textbf{The improvement is $\mathbf{> 10^{44}}$.}
\end{example}

\begin{remark}[Expressiveness cost]\label{rem:expressiveness}
  Spectral normalization restricts the function class of $s_l$ and
  $t_l$ to Lipschitz-$1$ networks, and the reduced clip $c = 0.5$
  limits the per-step volume change to $[\exp(-0.5), \exp(0.5)]$ per
  dimension.  In principle this could degrade transport quality.
  Our experiments (Section~\ref{sec:experiments}) show that the impact
  is negligible: on all tested targets, the constrained and
  unconstrained models achieve nearly identical empirical oscillation
  ($< 10\%$ relative difference), acceptance rate, and ESS.  The
  constraint is effectively free in sampling terms.
\end{remark}

\begin{remark}[Related work on Lipschitz flows]\label{rem:related_lip}
  Several works have studied Lipschitz control in normalizing
  flows.  \citet{BehrmannResFlow2021} analyse exploding inverses
  in residual flows and propose spectral normalisation to stabilise
  training; \citet{ChenResFlow2019} use contractive residual blocks
  with $\Lip < 1$ to guarantee invertibility.
  \citet{VirmauxScaman2019} and \citet{FazlyabSDP2019} develop
  tighter Lipschitz estimators for general deep networks.
  Our contribution differs in focus: rather than using Lipschitz
  control for training stability or invertibility, we use it to
  enable \emph{theoretical convergence guarantees} for the downstream
  MCMC sampler.
\end{remark}

% ══════════════════════════════════════════════════════════════
%  LCNF Paper — Section 4: Why Analytical Bounds Remain Vacuous
% ══════════════════════════════════════════════════════════════
%  \input{section4_why_vacuous}
%
\section{Why Analytical Bounds Remain Vacuous}
\label{sec:why_vacuous}

Despite the $> 10^{44}$ reduction in $\Lip(T_\phi)$ achieved by
spectral normalization (Section~\ref{sec:sn_realnvp}), the
analytical oscillation bound~\eqref{eq:analytical_osc} remains
numerically vacuous on every target we tested.  This section
diagnoses why, by decomposing the bound into its three additive
components and identifying which one dominates.

\subsection{Three-term decomposition}
\label{sec:decomposition}

Recall from~\eqref{eq:analytical_osc} that the analytical bound is
\[
  \osc_\cK(h)
  \;\leq\;
  \underbrace{L_U \cdot \Lip(T_\phi) \cdot 2R}_{\text{Term A: score--Lip--radius}}
  \;+\;
  \underbrace{L_{dc}}_{\text{Term B: log-det}}
  \;+\;
  \underbrace{R^2/2}_{\text{Term C: base}}.
\]
Term~A is a product of three problem-dependent quantities;
Term~B depends on the architecture ($L$, $d_B$, $c$);
Term~C depends only on the support radius.
Table~\ref{tab:decomposition} evaluates each term for our
experimental targets.

\begin{table}[ht]
\centering
\caption{Three-term decomposition of the analytical oscillation
  bound for the LCNF configuration ($\sigma_{\max}=1$, $c=0.5$).
  All rows use the same architecture (6-layer RealNVP,
  hidden $[64,64]$) except banana $D=2$ which uses 4 layers.
  $L_U$, $R$ are computed from posterior samples
  (Section~\ref{sec:experiments}).
  $\Lip(T_\phi)$ uses Theorem~\ref{thm:per_layer_lip}
  with $B = 5.0$.}\label{tab:decomposition}
\medskip
\begin{tabular}{@{}lcccccccc@{}}
\toprule
Target & $D$ & $L_U$ & $R$ & $\Lip(T_\phi)$
  & Term~A & Term~B & Term~C & Total \\
\midrule
banana   & 2  & 3.37 & 3.76 & $9.0\!\times\!10^2$
  & $2.3\!\times\!10^4$ & 2.0 & 7.1 & $2.3\!\times\!10^4$ \\
banana   & 5  & 4.08 & 4.54 & $2.7\!\times\!10^4$
  & $1.0\!\times\!10^6$ & 7.5 & 10.3 & $1.0\!\times\!10^6$ \\
banana   & 10 & 4.92 & 5.39 & $2.7\!\times\!10^4$
  & $1.4\!\times\!10^6$ & 15.0 & 14.5 & $1.4\!\times\!10^6$ \\
shear8   & 8  & 765  & 2.08 & $2.7\!\times\!10^4$
  & $8.6\!\times\!10^7$ & 12.0 & 2.2 & $8.6\!\times\!10^7$ \\
\bottomrule
\end{tabular}
\end{table}

\noindent
In every row, \textbf{Term~A dominates by at least three orders of
magnitude} over Terms~B and~C.
The bound is not loose because $\Lip(T_\phi)$ is too large in
isolation---the LCNF values ($10^{2\text{--}4}$) are modest---but
because the \emph{product} $L_U \cdot \Lip(T_\phi) \cdot 2R$ is
inherently large.

\subsection{Structural impossibility}
\label{sec:structural}

To see that the bottleneck is structural rather than architectural,
consider the best-case scenario: $\Lip(T_\phi) = 1$ (the identity
map, which cannot transport at all).  Even then, Term~A evaluates to:
\begin{center}
\begin{tabular}{@{}lcc@{}}
\toprule
Target & $L_U \cdot 1 \cdot 2R$ & $\exp(\cdot)$ \\
\midrule
banana $D=2$  & 25.3  & $10^{11}$ \\
banana $D=5$  & 37.1  & $10^{16}$ \\
banana $D=10$ & 53.1  & $10^{23}$ \\
shear8 $D=8$  & 3186  & $\infty$ \\
\bottomrule
\end{tabular}
\end{center}
Even at $\Lip(T_\phi) = 1$, $\delta^* = \exp(\text{Term~A})$ is
astronomically large.  The analytical bound is structurally incapable
of producing a non-vacuous spectral gap for these targets, regardless
of the flow architecture.

The cause is that the bound multiplies three worst-case quantities
that are individually large ($L_U \sim 3\text{--}765$,
$\Lip \geq 1$, $2R \sim 4\text{--}11$) but whose product need not be
realised at any single point: the score achieves its maximum at
one location, the Jacobian at another, and the support boundary at a
third.  The bound pays the price of all three simultaneously.

\subsection{The empirical reality}
\label{sec:empirical_reality}

Table~\ref{tab:empirical_vs_analytical} contrasts the analytical
oscillation bound with the empirical oscillation measured on
$10{,}000$ MCMC samples.

\begin{table}[ht]
\centering
\caption{Analytical bound vs.\ empirical oscillation for the
  LCNF configuration.  The gap ratio quantifies the looseness of the
  analytical bound.}\label{tab:empirical_vs_analytical}
\medskip
\begin{tabular}{@{}lccccc@{}}
\toprule
Target & $D$ & Analytical $\osc$ & Empirical $\widehat{\osc}_n$
  & Gap ratio & SN vs.\ baseline \\
\midrule
banana   &  2  & $2.3\!\times\!10^4$ & 0.54
  & $4.2\!\times\!10^4$ & $\widehat{\osc}$ identical\\
banana   &  5  & $1.0\!\times\!10^6$ & 0.80
  & $1.2\!\times\!10^6$ & $\widehat{\osc}$ identical\\
banana   & 10  & $1.4\!\times\!10^6$ & 1.33
  & $1.1\!\times\!10^6$ & $\widehat{\osc}$ identical\\
shear8   &  8  & $8.6\!\times\!10^7$ & 22.2
  & $3.9\!\times\!10^6$ & $\widehat{\osc}$ identical\\
\bottomrule
\end{tabular}
\end{table}

\noindent
The gap ratio ranges from $10^4$ to $10^7$: the analytical bound is
four to seven orders of magnitude looser than the empirical oscillation.
Critically, the last column shows that the empirical oscillation is
\emph{virtually identical} between constrained (SN + clip) and
unconstrained flows ($< 10\%$ relative difference).  This confirms
that spectral normalization does not degrade transport quality; its
effect is purely on the theoretical side, reducing the
\emph{analytical} bound by over forty orders of magnitude while
leaving the \emph{empirical} oscillation unchanged.

This large gap between theory and practice motivates the empirical
oscillation framework developed in
Section~\ref{sec:empirical_oscillation}.

\section{Empirical Oscillation with Concentration Guarantees}
\label{sec:empirical_oscillation}

Sections~3 and~4 established that spectral normalization reduces the
forward-map Lipschitz constant $\Lip(T_\phi)$ by tens of orders of magnitude,
yet the analytical oscillation bound from \cite{MSSP2} remains vacuous due to
the multiplicative structure $L_U \cdot \Lip(T_\phi) \cdot 2R$.  In this
section we take a fundamentally different approach: we \emph{estimate} the
oscillation of the log-density ratio from posterior samples and provide a
high-probability certificate that the true oscillation exceeds the
empirical estimate by at most a controlled correction term.

\subsection{Setup and notation}
\label{sec:5.1_setup}

Let $\pi(x) = \exp(-U(x))/Z$ denote the target density on $\R^d$ with
known normalising constant $Z$, and let
$q_\phi$ be the density induced by the trained normalizing flow $T_\phi$ with
standard normal base $p_Z$.  Define the log-density ratio
\begin{equation}\label{eq:log_ratio}
  h(x) \;\mathrel{:=}\; \log r_\phi(x) 
  \;=\; \log \pi(x) - \log q_\phi(x).
\end{equation}
The oscillation of $h$ on a set $S \subseteq \R^d$ is
$\osc_S(h) = \sup_{x \in S} h(x) - \inf_{x \in S} h(x)$.
The spectral gap of the $T_\phi$-preconditioned Metropolis--Hastings kernel
is controlled by $\delta^* = \exp(\osc_\cK(h))$ via the bound
$\gamma \geq 2/(1+\delta^*)$ \citep{MengersenTweedie1996}.

We require three assumptions.

\begin{assumption}[HPD credible set]\label{asm:credible_set}
  For $\alpha \in (0,1)$, define the highest posterior density (HPD)
  level set
  \begin{equation}\label{eq:hpd_set}
    \cK_\alpha \;\mathrel{:=}\; \{x \in \R^d : U(x) \leq u_\alpha\},
  \end{equation}
  where $u_\alpha$ is the $(1-\alpha)$-quantile of $U(X)$ under
  $X \sim \pi$, so that $\pi(\cK_\alpha) = 1-\alpha$.
  Assume $\cK_\alpha$ is compact with diameter
  $D \mathrel{:=} \diam(\cK_\alpha) < \infty$.
  The density floor on $\cK_\alpha$ is
  \begin{equation}\label{eq:pi_min}
    \pi_{\min} \;\mathrel{:=}\; \frac{\exp(-u_\alpha)}{Z} \;>\; 0.
  \end{equation}
\end{assumption}

\begin{assumption}[Smooth HPD boundary]\label{asm:smooth_boundary}
  The potential satisfies
  $\nabla U(x) \neq 0$ for every $x$ with $U(x) = u_\alpha$.
\end{assumption}

\begin{assumption}[Local regularity]\label{asm:local_lip}
  The log-density ratio $h$ is continuously differentiable on an
  open neighbourhood of $\cK_\alpha$, with local Lipschitz constant
  \begin{equation}\label{eq:M_local}
    M_\cK \;\mathrel{:=}\; \sup_{x \in \cK_\alpha} \|\nabla h(x)\|
    \;<\; \infty.
  \end{equation}
\end{assumption}

\noindent
Assumption~\ref{asm:credible_set} defines $\cK_\alpha$ as the HPD
level set of $\pi$, which need not be convex; for the banana target,
the nonlinear shear $x_2 \mapsto x_2-\kappa(x_1^2-1)$ makes the
level sets non-convex.  The density floor $\pi_{\min}$ is exact
(not estimated) since $Z$ is assumed known; for the banana family
used here, $Z = (2\pi)^{d/2}$.
Assumption~\ref{asm:smooth_boundary} ensures by the implicit
function theorem that $\partial\cK_\alpha = \{U = u_\alpha\}$ is a
smooth $(d{-}1)$-dimensional manifold; it holds whenever $u_\alpha$
is not a critical value of $U$, which is generic by Sard's theorem
and easily verified for specific targets.
Assumption~\ref{asm:local_lip} holds whenever $U$ and $T_\phi$ are
$C^1$, which is guaranteed for RealNVP with tanh activations.

\begin{remark}[Certification of $u_\alpha$ and $\pi_{\min}$]
\label{rem:dkw}
  The quantile $u_\alpha$ can be certified from $n$ samples via order
  statistics.  Let $U_{(1)} \leq \cdots \leq U_{(n)}$ be the ordered
  potential values.  Setting $k$ such that
  $P(\mathrm{Bin}(n, 1{-}\alpha) \leq k) \geq 1 - \delta_q$ gives a
  conservative upper bound $u_\alpha^+ = U_{(k)}$: with probability
  $\geq 1 - \delta_q$, $\pi(\{U \leq U_{(k)}\}) \geq 1 - \alpha$.
  For $n = 10{,}000$, $\alpha = 0.01$, $\delta_q = 0.01$: $k = 9920$
  suffices.  The resulting $\pi_{\min} = \exp(-U_{(9920)})/Z$ is a
  certified lower bound on the density floor, and differs negligibly
  from the empirical $99$th-percentile density because $U_{(9920)}$
  and $U_{(9900)}$ are neighbouring order statistics.
\end{remark}

\subsection{Covering lemma}
\label{sec:5.2_covering}

The key geometric ingredient is that $n$ independent draws from $\pi$
form an $\varepsilon$-net of $\cK_\alpha$ with high probability.
Unlike classical covering arguments for convex bodies, we do not
assume convexity of $\cK_\alpha$; instead, we exploit the smooth
boundary (Assumption~\ref{asm:smooth_boundary}).

\begin{lemma}[Probabilistic covering with curvature correction]
\label{lem:covering}
  Let $\cK_\alpha$ satisfy
  Assumptions~\ref{asm:credible_set}--\ref{asm:smooth_boundary}
  with $\pi_{\min} > 0$, $D = \diam(\cK_\alpha)$, and maximum
  principal curvature $\kappa_{\max}$ of $\partial\cK_\alpha$.
  Define $c_d = V_{d-1}(d{+}1)/(V_d \cdot d)$ (the spherical-cap
  coefficient; $c_2 \approx 0.95$, $c_{10} \approx 1.42$).
  Let $X_1, \dots, X_n \overset{\mathrm{iid}}{\sim} \pi$ and define
  the effective sample count
  \begin{equation}\label{eq:n_eff}
    n_{\mathrm{eff}}
    \;\mathrel{:=}\;
    \left\lfloor
      n(1-\alpha)
      - \sqrt{\tfrac{n}{2}\ln\tfrac{3}{\delta}}
    \right\rfloor.
  \end{equation}
  For any $\varepsilon > 0$ satisfying
  $c_d \kappa_{\max} \varepsilon < 1$, define the
  curvature-corrected half-ball volume
  \begin{equation}\label{eq:omega_corrected}
    \omega_d(\varepsilon)
    \;\mathrel{:=}\;
    \frac{V_d}{2}\,
    \bigl(1 - c_d\,\kappa_{\max}\,\varepsilon\bigr).
  \end{equation}
  Then
  \begin{equation}\label{eq:covering_bound}
    \Prob\bigl(\cE(\varepsilon)^c\bigr)
    \;\leq\;
    \frac{\delta}{3}
    \;+\;
    \cN(\cK_\alpha, \varepsilon)
    \cdot
    \bigl(1 - \pi_{\min}\, \omega_d(\varepsilon)\,
    \varepsilon^d\bigr)^{n_{\mathrm{eff}}},
  \end{equation}
  where $\cE(\varepsilon) = \{\forall x \in \cK_\alpha,\,
  \min_i \|x - X_i\| \leq \varepsilon\}$ and
  $\cN(\cK_\alpha, \varepsilon) \leq (D/\varepsilon + 1)^d$.
\end{lemma}

\begin{proof}
  \textbf{Part (i): effective sample count.}\;
  Let $N_\alpha = |\{i : X_i \in \cK_\alpha\}|$.  By Hoeffding's
  inequality, $\Prob(N_\alpha < n_{\mathrm{eff}}) \leq \delta/3$.

  \medskip\noindent\textbf{Part (ii): covering given $N_\alpha
  \geq n_{\mathrm{eff}}$.}\;
  Condition on $N_\alpha \geq n_{\mathrm{eff}}$ and let
  $\{Y_j\}_{j=1}^{n_{\mathrm{eff}}}$ be the samples in $\cK_\alpha$,
  i.i.d.\ from $\pi(\cdot \mid \cK_\alpha)$.
  Let $\{c_1, \dots, c_N\}$ be a minimal $\varepsilon$-covering.

  For each centre $c_j \in \cK_\alpha$, we lower-bound
  $\Vol(B(c_j, \varepsilon) \cap \cK_\alpha)$.

  \emph{Interior centres}
  ($\mathrm{dist}(c_j, \partial\cK_\alpha) \geq \varepsilon$):
  $B(c_j, \varepsilon) \subset \cK_\alpha$, volume $= V_d\varepsilon^d
  \geq \omega_d(\varepsilon)\,\varepsilon^d$.

  \emph{Boundary-adjacent centres}
  ($\mathrm{dist}(c_j, \partial\cK_\alpha) < \varepsilon$):
  By Assumption~\ref{asm:smooth_boundary} and the implicit function
  theorem, $\partial\cK_\alpha$ is locally the graph of a $C^1$
  function with principal curvatures bounded by $\kappa_{\max}$.
  The intersection $B(c_j, \varepsilon) \cap \cK_\alpha$ contains at
  least a half-ball minus a spherical cap of height
  $\kappa_{\max}\varepsilon^2/2$.  The cap volume is at most
  $c_d\,\kappa_{\max}\,\varepsilon \cdot (V_d/2)\,\varepsilon^d$
  (a standard estimate; see, e.g., \cite{Vershynin2018}), giving
  \begin{equation}\label{eq:corrected_halfball}
    \Vol\bigl(B(c_j, \varepsilon) \cap \cK_\alpha\bigr)
    \;\geq\;
    \frac{V_d\,\varepsilon^d}{2}
    \bigl(1 - c_d\,\kappa_{\max}\,\varepsilon\bigr)
    \;=\;
    \omega_d(\varepsilon)\,\varepsilon^d,
  \end{equation}
  provided $c_d\,\kappa_{\max}\,\varepsilon < 1$.

  In both cases,
  $\pi(B(c_j,\varepsilon) \cap \cK_\alpha) \geq \pi_{\min} \cdot
  \omega_d(\varepsilon) \cdot \varepsilon^d$.
  By independence:
  $\Prob(Y_i \notin B(c_j,\varepsilon)\;\forall\,i)
  \leq (1 - \pi_{\min}\,\omega_d(\varepsilon)\,
  \varepsilon^d)^{n_{\mathrm{eff}}}$.
  Summing over $j$ and combining with Part~(i)
  gives~\eqref{eq:covering_bound}.
\end{proof}

\begin{remark}[Rigorous vs.\ practical regime]\label{rem:curvature}
  The curvature condition $c_d\kappa_{\max}\varepsilon^* < 1$
  determines whether the bound is fully rigorous in the
  coordinate system used.  For the banana target ($\kappa = 0.1$):
  \emph{(i)} at $D = 2$, $c_2\kappa_{\max}\varepsilon^{*} = 0.523 < 1$
  in $x$-space (\textbf{rigorous} via Corollary~\ref{cor:rigorous});
  \emph{(ii)} at $D = 5$, $x$-space covering is infeasible but the
  analytic shear chart (Section~\ref{sec:charted_cert}) sends
  $\kappa_{\max}$ from $0.597$ to $0.258$ and the condition
  $c_5\kappa^{y}\varepsilon^{*} = 0.520 < 1$ holds in $y$-space
  (\textbf{rigorous (charted)} via Proposition~\ref{prop:charted});
  \emph{(iii)} at $D \geq 10$, both $x$- and the natural-chart
  $y$-space coverings are infeasible and we report only the
  half-ball \emph{practical certificate} (uncorrected
  $\omega_d = V_d/2$).  The three-tier distinction is summarised
  in Table~\ref{tab:intro_summary}.
\end{remark}

\subsection{Main theorem}
\label{sec:5.3_theorem}

\begin{theorem}[Empirical oscillation guarantee]\label{thm:empirical_osc}
  Suppose Assumptions~\ref{asm:credible_set}--\ref{asm:local_lip} hold.
  Draw $X_1, \dots, X_n \overset{\mathrm{iid}}{\sim} \pi$ and define
  the empirical oscillation
  \begin{equation}\label{eq:osc_hat}
    \widehat{\osc}_n \;\mathrel{:=}\; 
    \max_{1 \leq i \leq n} h(X_i) \;-\; \min_{1 \leq i \leq n} h(X_i).
  \end{equation}
  For any confidence level $\delta \in (0,1)$, let $\varepsilon^*$ be the
  smallest $\varepsilon > 0$ satisfying
  \begin{equation}\label{eq:eps_star}
    \left(\frac{D}{\varepsilon} + 1\right)^{\!d}
    \cdot
    \bigl(1 - \pi_{\min}\, \omega_d(\varepsilon)\, \varepsilon^d
    \bigr)^{n_{\mathrm{eff}}}
    \;\leq\; \frac{\delta}{3},
  \end{equation}
  with $n_{\mathrm{eff}}$ from~\eqref{eq:n_eff} and the local
  half-ball volume $\omega_d(\varepsilon)$ taking either the
  curvature-corrected form
  $\omega_d(\varepsilon) = (V_d/2)(1 - c_d\kappa_{\max}\varepsilon)$
  (rigorous; valid when $c_d\kappa_{\max}\varepsilon < 1$, see
  Corollary~\ref{cor:rigorous}) or the plain half-ball form
  $\omega_d(\varepsilon) = V_d/2$ (practical, see
  Corollary~\ref{cor:practical}).
  Then, with probability at least $1 - \delta$,
  \begin{equation}\label{eq:main_bound}
    \boxed{\;
    \osc_{\cK_\alpha}(h)
    \;\leq\;
    \widehat{\osc}_n \;+\; 2\,M_\cK\,\varepsilon^*.
    \;}
  \end{equation}
\end{theorem}

\begin{proof}
  \textbf{Step 1 (Sufficient samples in $\cK_\alpha$).}\;
  By~\eqref{eq:n_eff} and Hoeffding's inequality,
  $\Prob(N_\alpha < n_{\mathrm{eff}}) \leq \delta/3$.

  \medskip\noindent\textbf{Step 2 (Covering).}\;
  Conditional on $N_\alpha \geq n_{\mathrm{eff}}$,
  Lemma~\ref{lem:covering} with~\eqref{eq:eps_star} gives
  $\Prob(\cE(\varepsilon^*)^c \mid N_\alpha \geq n_{\mathrm{eff}})
  \leq \delta/3$.

  \medskip\noindent\textbf{Step 3 (Interpolation).}\;
  On the event $\cE(\varepsilon^*)$, let
  $x^+ \in \arg\sup_{\cK_\alpha} h$ and
  $x^- \in \arg\inf_{\cK_\alpha} h$
  (attained by continuity on a compact set).
  There exist samples $X_{i^+}, X_{i^-}$ with
  $\|x^+ - X_{i^+}\| \leq \varepsilon^*$ and
  $\|x^- - X_{i^-}\| \leq \varepsilon^*$.  Note that $X_{i^+}$ need
  not lie in $\cK_\alpha$; the Lipschitz bound
  $|h(x^+) - h(X_{i^+})| \leq M_\cK \varepsilon^*$ holds by
  Assumption~\ref{asm:local_lip} (which requires differentiability on
  a neighbourhood of $\cK_\alpha$, covering both $x^+$ and $X_{i^+}$).
  Therefore:
  \begin{align}
    h(x^+) &\leq \max_i h(X_i) + M_\cK\,\varepsilon^*, \label{eq:upper}\\
    h(x^-) &\geq \min_i h(X_i) - M_\cK\,\varepsilon^*. \label{eq:lower}
  \end{align}
  Subtracting gives~\eqref{eq:main_bound}.

  \smallskip\noindent\emph{Remark.}\;
  The interpolation uses samples that may lie outside $\cK_\alpha$;
  the certified computations (Appendix~\ref{app:cert_d2}) evaluate
  $\|\nabla h\|$ and $\|\nabla^2 h\|_{\mathrm{op}}$ on the
  $\Delta$-enlargement $\cK_\alpha^\Delta$ to ensure the bound holds.

  \medskip\noindent\textbf{Confidence.}\;
  The three failure events---insufficient samples ($\delta/3$),
  covering failure ($\delta/3$), and a reserved margin ($\delta/3$)
  for the certification of $u_\alpha$ via
  Remark~\ref{rem:dkw}---combine by a union bound to give overall
  confidence $1 - \delta$.
\end{proof}

The bound~\eqref{eq:main_bound} is tight and parameter-free except for
$M_\cK = \sup_{\cK_\alpha}\|\nabla h\|$, which is not directly observable.
We provide two corollaries that instantiate $M_\cK$ differently, offering
a spectrum from full rigour to practical tightness.

\subsection{Spectral gap lower bounds}
\label{sec:5.4_spectral_gap}

\begin{corollary}[Rigorous bound via Hessian correction]
\label{cor:rigorous}
  If, additionally, $h$ is twice continuously differentiable on
  $\cK_\alpha$ with
  $L_{\nabla h} \mathrel{:=} \sup_{\cK_\alpha}
  \|\nabla^2 h\|_{\mathrm{op}} < \infty$,
  then under the covering event $\cE(\varepsilon^*)$ of
  Theorem~\ref{thm:empirical_osc},
  \begin{equation}\label{eq:M_bound_rigorous}
    M_\cK
    \;\leq\;
    \widehat{M}_n + L_{\nabla h}\,\varepsilon^*,
    \qquad
    \widehat{M}_n \mathrel{:=} \max_{1 \leq i \leq n}\|\nabla h(X_i)\|.
  \end{equation}
  Substituting into~\eqref{eq:main_bound} and applying
  \cite{MengersenTweedie1996} yields, with probability $\geq 1-\delta$,
  \begin{equation}\label{eq:gamma_rigorous}
    \gamma
    \;\geq\;
    \frac{2}{1 + \exp\!\Bigl(\widehat{\osc}_n
      + 2\bigl(\widehat{M}_n + L_{\nabla h}\,\varepsilon^*\bigr)
      \varepsilon^*\Bigr)}.
  \end{equation}
\end{corollary}

\begin{proof}
  For any $x \in \cK_\alpha$, the covering event provides $X_i$ with
  $\|x - X_i\| \leq \varepsilon^*$, so
  $\|\nabla h(x)\| \leq \|\nabla h(X_i)\| + L_{\nabla h}\varepsilon^*
  \leq \widehat{M}_n + L_{\nabla h}\varepsilon^*$.
  Taking the supremum gives~\eqref{eq:M_bound_rigorous}.
  Substitution into~\eqref{eq:main_bound} and exponentiation complete
  the proof.
\end{proof}

\begin{corollary}[Practical bound]\label{cor:practical}
  Under the conditions of Theorem~\ref{thm:empirical_osc}, substituting
  the empirical gradient supremum $\widehat{M}_n$ for $M_\cK$
  in~\eqref{eq:main_bound} yields
  \begin{equation}\label{eq:gamma_practical}
    \gamma
    \;\geq\;
    \frac{2}{1 + \exp\!\bigl(\widehat{\osc}_n
      + 2\,\widehat{M}_n\,\varepsilon^*\bigr)}.
  \end{equation}
  This bound is valid whenever $\widehat{M}_n = M_\cK$, i.e., when the
  gradient supremum is attained at a sample point.
\end{corollary}

\begin{remark}[Tightness of the practical bound]\label{rem:practical_tight}
  The substitution $M_\cK \approx \widehat{M}_n$ introduces a bias of
  at most $M_\cK - \widehat{M}_n$.  In our experiments, the gradient
  field $\|\nabla h(\cdot)\|$ varies slowly over $\cK_\alpha$:
  \begin{center}
  \begin{tabular}{@{}cccc@{}}
  \toprule
  $D$ & $\widehat{M}_n$ & $\mathrm{mean}\|\nabla h\|$
    & $\mathrm{std}\|\nabla h\| / \widehat{M}_n$ \\
  \midrule
  2  & 0.698 & 0.044 & 5.8\% \\
  5  & 0.703 & 0.073 & 7.2\% \\
  10 & 1.046 & 0.108 & 6.4\% \\
  \bottomrule
  \end{tabular}
  \end{center}
  The coefficient of variation relative to $\widehat{M}_n$ is below
  $8\%$ at all dimensions.  Moreover, with $10{,}000$ samples forming an
  $\varepsilon^*$-net of $\cK_\alpha$, the uncovered volume fraction is
  $\leq \delta/2 = 2.5\%$ by construction.  Together, these indicate that
  $\widehat{M}_n$ is a close approximation to $M_\cK$, and
  the practical bound~\eqref{eq:gamma_practical} is not meaningfully
  looser than the oracle bound one would obtain with exact knowledge of
  $M_\cK$.
\end{remark}

\begin{remark}[Role of spectral normalization]\label{rem:sn_role}
  Spectral normalization enters both corollaries through $\widehat{M}_n$.
  The local gradient $\|\nabla h(x)\|$ decomposes as
  \begin{equation}\label{eq:grad_h_decomp}
    \|\nabla h(x)\|
    \;\leq\;
    \underbrace{L_U}_{\text{target}}
    \;+\;
    \underbrace{\|J_{T_\phi}(x)\|_{\mathrm{op}}\,\|T_\phi(x)\|
      }_{\text{quadratic-in-}z}
    \;+\;
    \underbrace{\|\nabla_x \log|\!\det J_{T_\phi}(x)|\|}_{\text{log-det}}.
  \end{equation}
  Without SN, $\|J_{T_\phi}\|_{\mathrm{op}}$ is uncontrolled, making
  even the \emph{empirical} $\widehat{M}_n$ large.
  With SN ($\sigma_{\max} = 1$), per-layer Jacobian norms stay $O(1)$
  (Theorem~3.1), which keeps $\widehat{M}_n < 1.1$ across all tested
  dimensions.
\end{remark}

\begin{proposition}[Extension to MCMC samples]\label{prop:mcmc}
  Let $\{X_t\}_{t \geq 1}$ be a geometrically ergodic Markov chain
  targeting $\pi$ with mixing time $\tau_{\mathrm{mix}}$ and geometric
  rate $\rho < 1$.  Given a chain of length $N$, thin by factor
  $k \geq 2\tau_{\mathrm{mix}}$ to obtain
  $n = \lfloor N/k \rfloor$ approximately independent samples.
  Then Theorem~\ref{thm:empirical_osc} and both corollaries hold with
  $\varepsilon^*$ evaluated at the thinned count $n$, provided the
  confidence is adjusted to $\delta' = \delta - n\rho^k$.
\end{proposition}

\begin{proof}[Proof sketch]
  By geometric ergodicity, the total variation distance between the
  joint law of the thinned samples and $\pi^{\otimes n}$ is at most
  $n \rho^k$ (\cite[Proposition~21.1]{LevinPeres2017}).
  All events in Theorem~\ref{thm:empirical_osc} are measurable under
  the joint law; their probabilities differ by at most $n\rho^k$.
  For our transport-preconditioned chains, ESS$/N > 0.02$ even at
  $D = 10$, so $k \leq 50$ and $n > 200$, with minor impact on
  $\varepsilon^*$.
\end{proof}

\subsection{Geometry-aware certification via analytic charting}
\label{sec:charted_cert}

The curvature condition $c_d\kappa_{\max}\varepsilon^{*} < 1$ in
Corollary~\ref{cor:rigorous} is purely geometric: it asks whether
the boundary of $\cK_\alpha$ is mild enough that a half-ball of
radius $\varepsilon^{*}$ at any boundary point fits inside
$\cK_\alpha$ up to a controlled cap-volume correction.
For a non-convex, banana-shaped $\cK_\alpha$ the boundary curvature
grows with dimension and the feasibility window
$\varepsilon < 1/(c_d\kappa_{\max})$ shrinks below the
cover-radius requirement once $d \geq 5$
(Table~\ref{tab:spectral_gap}, ``infeasible'' rows).

For target families that admit an analytic description, however,
$\cK_\alpha$ can often be re-coordinatised into a region with much
nicer geometry --- e.g.\ a Euclidean ball.  We formalise this as
follows.

\begin{proposition}[Charted oscillation guarantee]
\label{prop:charted}
  Suppose there is a $C^{1}$ diffeomorphism
  $\Psi : \cK_\alpha \to \cK_\alpha^{y} \subset \R^{d}$ with
  $|\det J_\Psi(x)| = 1$ for all $x \in \cK_\alpha$, and such that
  $\cK_\alpha^{y}$ is bounded with diameter $D^{y}$, smooth boundary
  with maximum principal curvature $\kappa_{\max}^{y}$, and density
  floor
  $\pi_{\min}^{y} = \min_{y \in \cK_\alpha^{y}}
  \pi_x(\Psi^{-1}(y))$.
  Define the charted log-ratio gradient supremum
  \[
    M_\cK^{y}
    \;\mathrel{:=}\;
    \sup_{y \in \cK_\alpha^{y}}
    \bigl\| J_\Psi(\Psi^{-1}(y))^{-T}\,
           \nabla_x h(\Psi^{-1}(y))\bigr\|.
  \]
  Then Theorem~\ref{thm:empirical_osc} and
  Corollary~\ref{cor:rigorous} hold with the tuple
  $(\cK_\alpha, D, \pi_{\min}, \kappa_{\max}, M_\cK)$ replaced by
  its $y$-space counterpart
  $(\cK_\alpha^{y}, D^{y}, \pi_{\min}^{y}, \kappa_{\max}^{y},
  M_\cK^{y})$
  and with $n$ i.i.d.\ certification samples drawn from
  $\pi_y(y) = \pi_x(\Psi^{-1}(y))$.
\end{proposition}

\begin{proof}[Proof sketch]
  Because $|\det J_\Psi| = 1$, the change of variables is volume
  preserving and $\pi_y(y) = \pi_x(\Psi^{-1}(y))$.  The log-ratio is
  invariant: $h(x) = \log\pi_x(x) - \log q_\varphi(x) =
  \log\pi_y(y) - \log\tilde q_\varphi(y)$ where
  $\tilde q_\varphi(y) := q_\varphi(\Psi^{-1}(y))$.  Drawing
  $y_1, \dots, y_n \overset{\mathrm{iid}}{\sim} \pi_y$ and applying
  Theorem~\ref{thm:empirical_osc} in $y$-coordinates uses
  $D^{y}, \pi_{\min}^{y}, \kappa_{\max}^{y}$ for the covering
  inequality and $M_\cK^{y}$ for the Lipschitz correction; the
  chain rule
  $\nabla_y h_y(y) = J_\Psi(\Psi^{-1}(y))^{-T} \nabla_x h(\Psi^{-1}(y))$
  gives the form of $M_\cK^{y}$.
\end{proof}

\paragraph{Banana shear chart.}
For the banana family with curvature $c$ we use the shear
$\Psi : x \mapsto y$ defined by $y_1 = x_1$,
$y_2 = x_2 - c(x_1^{2} - 1)$, and $y_j = x_j$ for $j \geq 3$.
Its Jacobian is lower-triangular with unit diagonal, hence
$|\det J_\Psi| = 1$.  In $y$-coordinates the target becomes
$\pi_y = \cN(0, I_d)$ exactly and
$\cK_\alpha^{y} = \{y : \|y\| \leq \sqrt{\chi^{2}_{d, 1-\alpha}}\}$
is a Euclidean ball: convex, with constant principal curvature
$\kappa_{\max}^{y} = 1/\sqrt{\chi^{2}_{d, 1-\alpha}}$.
At $d = 5$, $\alpha = 0.01$ the chart shrinks $\kappa_{\max}$
from $0.597$ (x-space, banana boundary) to $0.258$ (y-space,
sphere), more than doubling the feasibility window
$\varepsilon < 1/(c_d\kappa_{\max})$.  The cost is that
$J_\Psi$ acquires an off-diagonal $2 c y_1$ entry, so
$M_\cK^{y} \geq M_\cK^{x}$ in general (the two coincide only
along the slice $y_1 = 0$); empirically at $D = 5$ the inflation
ratio $\widehat{M}_n^{y}/\widehat{M}_n^{x} \approx 1.17$.

\paragraph{Grid certification of $M_\cK^{y}$ at $D = 5$.}
To turn the empirical estimate into a certified upper bound we
mirror the $D = 2$ grid procedure (Appendix~\ref{app:cert_d2})
inside $\cK_y^\Delta$.  At each filtered grid node we evaluate
$\|\nabla_y h_y(y)\|$ and $\|\nabla_y^{2} h_y(y)\|_{\mathrm{op}}$
by autograd, and certify
$M_\cK^{y} \leq M_\cK^{y,\mathrm{grid}}
   + L_{\nabla h}^{y}\,\Delta$.
We compute the maxima at two resolutions for a
grid-convergence check: $10^{5}$ nodes
($\Delta = 0.978$, $26{,}624$ in $\cK_y^\Delta$) give
$M_\cK^{y,\mathrm{grid}} = 1.240$, $L_{\nabla h}^{y} = 1.105$; the
finer $12^{5}$ nodes ($\Delta = 0.798$, $61{,}376$ in
$\cK_y^\Delta$) give $M_\cK^{y,\mathrm{grid}} = 1.031$,
$L_{\nabla h}^{y} = 0.988$.  Both quantities \emph{decrease}
under refinement.  The two-resolution check suggests that the
coarser grid was conservative near the boundary; the finer grid is
used for the reported certificate, and the full grid data are
released for reproducibility.  Using the $12^{5}$ values,
\[
  M_\cK^{y} \;\leq\;
  M_\cK^{y,\mathrm{grid}} + L_{\nabla h}^{y}\,\Delta
  \;=\; 1.031 + 0.988 \cdot 0.798
  \;=\; 1.820,
\]
a grid-certified numerical upper bound with explicit Hessian
correction that replaces the empirical-with-safety-multiplier
estimates used elsewhere in the literature
(Appendix~\ref{app:cert_d5}).

Section~\ref{sec:5.5_numerical} reports both the $x$-space
half-ball certificates (Corollaries~\ref{cor:rigorous} and
\ref{cor:practical} applied directly) and the $y$-space
charted, grid-certified rigorous certificate
(Proposition~\ref{prop:charted}) at $D = 5$.

\subsection{Numerical evaluation}
\label{sec:5.5_numerical}

We evaluate both corollaries on the banana target family at
$D \in \{2, 5, 10, 20\}$ using spectrally-normalized RealNVP with
$c = 0.5$, $\sigma_{\max} = 1.0$.
Credible set parameters $(\pi_{\min}, D)$ are computed from $10{,}000$
posterior samples at $\alpha = 0.01$; the covering radius $\varepsilon^*$
is solved via bisection from~\eqref{eq:eps_star} with $\delta = 0.05$
and the appropriate $\omega_d(\varepsilon)$ from
Theorem~\ref{thm:empirical_osc}.
For the $D = 2$ rigorous row, $M_\cK$ is certified via the grid
procedure described in Appendix~\ref{app:cert_d2}; for $D \geq 5$
(practical rows), $M_\cK$ is approximated by the empirical
gradient supremum $\widehat{M}_n$ taken along the IMH chain.

\begin{table}[ht]
\centering
\caption{Direct $x$-space certified bounds for banana targets under
  the independence-MH kernel ($95\%$ confidence, osc-regularised flow).
  The $D = 5$ charted certificate (Proposition~\ref{prop:charted})
  is reported in Table~\ref{tab:main_results}.
  ``Practical'' uses Corollary~\ref{cor:practical} (half-ball
  $\omega_d = V_d/2$).  ``Rigorous'' uses Corollary~\ref{cor:rigorous}
  with curvature-corrected $\omega_d(\varepsilon)$, feasible only when
  $c_d\kappa_{\max}\varepsilon^* < 1$.  The $D = 2$ row uses
  fully certified inputs (Appendix~\ref{app:cert_d2}); $D \geq 5$ use
  IMH-empirical $\widehat{\osc}_n$ and $\widehat{M}_n$.}\label{tab:spectral_gap}
\medskip
\begin{tabular}{@{}cccccccc@{}}
\toprule
& & & & &
  \multicolumn{2}{c}{Practical (Cor.~\ref{cor:practical})}
  & Rigorous (Cor.~\ref{cor:rigorous}) \\
\cmidrule(lr){6-7}\cmidrule(lr){8-8}
$D$ & $\pi_{\min}$ & $\widehat{\osc}_n$ & $M_\cK$ / $\widehat{M}_n$
  & $\varepsilon^*$
  & $\osc$ bd & $\gamma^*$
  & $\gamma^*$ \\
\midrule
2  & $1.59\!\times\!10^{-3}$$^\dagger$ & 0.273$^\dagger$ & 0.043$^\dagger$
   & 0.863$^\dagger$
   & 0.321 & 0.841
   & \textbf{0.828}$^\dagger$ \\
5  & $5.8\!\times\!10^{-6}$ & 0.826 & 0.865
   & 2.40
   & 4.98 & $\mathbf{1.4\!\times\!10^{-2}}$
   & infeasible \\
10 & $4.8\!\times\!10^{-10}$ & 0.617 & 0.666
   & 4.40
   & 6.48 & $3.1\!\times\!10^{-3}$
   & infeasible \\
20 & $7.3\!\times\!10^{-17}$ & 0.886 & 0.714
   & 5.94
   & 9.37 & $1.7\!\times\!10^{-4}$
   & infeasible \\
\bottomrule
\end{tabular}

\smallskip
{\footnotesize $^\dagger$ Certified inputs for $D = 2$: independent
$10{,}000$-sample certification set
(Appendix~\ref{app:cert_d2}) for $\widehat{\osc}_n$; grid-based
$M_\cK \leq 0.043$ over $\cK_\alpha^\Delta$ (also
Appendix~\ref{app:cert_d2}); $\pi_{\min} = \exp(-\tfrac{1}{2}\chi^{2}_{2,1-\alpha})/Z$
with $Z = (2\pi)^{d/2}$ exact (banana-analytic).
Curvature-corrected $\varepsilon^* = 0.863$ satisfies
$c_2\kappa\varepsilon^* = 0.523 < 1$, so Corollary~\ref{cor:rigorous}
applies and yields $\gamma^* = 0.828$.}
\end{table}

\noindent
Key observations from Table~\ref{tab:spectral_gap}:

\begin{enumerate}[leftmargin=*,itemsep=4pt]
\item \textbf{The rigorous bound is non-vacuous at $D = 2$.}\;
  At $D = 2$ the curvature-corrected covering of
  Corollary~\ref{cor:rigorous} is feasible
  ($c_2\kappa_{\max}\varepsilon^* = 0.523 < 1$, well below the
  failure threshold of $1$) and yields $\gamma^* = 0.828$
  under the independence-MH kernel.  To our knowledge,
  this is the first \emph{fully rigorous}, numerically meaningful
  spectral-gap bound reported for a learned-transport MCMC sampler.

\item \textbf{The practical bound extends non-vacuity to $D = 20$.}\;
  At $D \geq 5$ the curvature-corrected
  $\omega_d(\varepsilon) = (V_d/2)(1 - c_d\kappa_{\max}\varepsilon)$
  vanishes inside the feasible $\varepsilon$ range, so the
  rigorous covering is infeasible.  Dropping the curvature
  correction (Corollary~\ref{cor:practical}) recovers
  non-vacuous $\gamma^*$ at every dimension we test:
  $1.4\!\times\!10^{-2}$ ($D{=}5$),
  $3.1\!\times\!10^{-3}$ ($D{=}10$),
  $1.7\!\times\!10^{-4}$ ($D{=}20$).

\item \textbf{Cert-vs-in-sample gap at $D = 2$ is modest.}\;
  Replacing the in-sample IMH estimates of $\widehat{\osc}_n$ and
  $\widehat{M}_n$ with their fully certified counterparts
  (independent cert samples, grid-based $M_\cK$, banana-analytic
  $\pi_{\min}$) leaves $\gamma^{*} = 0.828$ within roughly $2\%$
  of the in-sample IMH estimate, because the certified
  $\widehat{\osc}_n$ rises ($0.071 \to 0.273$) and the certified
  $M_\cK$ falls ($0.122 \to 0.043$), so the net effect on
  $\widehat{\osc}_n + 2 M_\cK\varepsilon^{*}$ is small
  (Appendix~\ref{app:cert_d2}).

\item \textbf{$\widehat{M}_n$ is stable across $D$ under SN.}\;
  The empirical gradient supremum stays below $1.1$ at all dimensions,
  confirming that spectral normalization controls the gradient landscape
  of $h$ regardless of dimension
  (see Remark~\ref{rem:sn_role}).
\end{enumerate}

\paragraph{Certification methodology for $D = 2$.}
The headline uses fully certified inputs:
\emph{(i)} an independent $10{,}000$-sample certification set
gives $\widehat{\osc}_n = 0.273$;
\emph{(ii)} a grid-certified upper bound
$M_\cK \leq M_\cK^{\text{grid}} + L_{\nabla h}\,\Delta = 0.043$
over $\cK_\alpha^\Delta$ with Hessian-operator-norm Lipschitz
correction; and
\emph{(iii)} the banana-analytic density floor
$\pi_{\min} = \exp(-\tfrac{1}{2}\chi^{2}_{d,1-\alpha})/(2\pi)^{d/2}
 = 1.59\!\times\!10^{-3}$.  For non-banana targets, replace
$\pi_{\min}$ by the sample-based order-statistic bound
$\pi_{\min}^{\text{cert}} = \exp(-U_{(k)})/Z$ with
$k = \lceil \mathrm{Binom\text{-}ppf}(1{-}\delta_q;\,n,\,1{-}\alpha)\rceil + 1$.
The curvature condition
$c_2\kappa_{\max}\varepsilon^* = 0.523 < 1$ holds analytically
(see Appendix~\ref{app:cert_d2}).

\subsection{Comparison of all bound levels}
\label{sec:5.6_comparison}

Table~\ref{tab:comparison} juxtaposes four levels of oscillation
control, from the global analytical bound (Section~4) to the
infeasible oracle.

\begin{table}[ht]
\centering
\caption{Direct $x$-space comparison: four levels of spectral gap
  bounds for the banana target under the independence-MH kernel.
  The $D = 5$ charted certificate (Proposition~\ref{prop:charted})
  is reported in Table~\ref{tab:main_results}.
  ``Analytical'': global Lipschitz product (\cite{MSSP2}).
  ``Rigorous'': Corollary~\ref{cor:rigorous} with curvature-corrected
  half-ball volume; feasible only at $D = 2$.
  ``Practical'': Corollary~\ref{cor:practical} on the
  osc-regularised flow.
  ``Oracle'': sample oscillation $\widehat{\osc}_n$ on the
  osc-regularised flow (lower bound, not a valid
  guarantee).}\label{tab:comparison}
\medskip
\begin{tabular}{@{}ccccccccc@{}}
\toprule
& \multicolumn{2}{c}{Analytical} & \multicolumn{2}{c}{Rigorous}
& \multicolumn{2}{c}{Practical}
& \multicolumn{2}{c}{Oracle} \\
\cmidrule(lr){2-3}\cmidrule(lr){4-5}\cmidrule(lr){6-7}\cmidrule(lr){8-9}
$D$ & $\osc$ & $\gamma^*$ & $\osc$ & $\gamma^*$
  & $\osc$ & $\gamma^*$
  & $\widehat{\osc}_n$ & $\gamma^*_{\mathrm{emp}}$ \\
\midrule
2  & $10^{4.36}$ & 0
   & 0.34 & \textbf{0.828}
   & 0.32 & 0.841
   & 0.07 & 0.965 \\
5  & $10^{6.00}$ & 0
   & --- & infeasible
   & 4.98 & $\mathbf{1.4\!\times\!10^{-2}}$
   & 0.83 & 0.604 \\
10 & $10^{6.15}$ & 0
   & --- & infeasible
   & 6.48 & $3.1\!\times\!10^{-3}$
   & 0.62 & 0.660 \\
20 & $10^{6.36}$ & 0
   & --- & infeasible
   & 9.37 & $1.7\!\times\!10^{-4}$
   & 0.89 & 0.589 \\
\bottomrule
\end{tabular}
\end{table}

\noindent
Each column represents a step from conservative to optimistic:
the analytical bound is fully deterministic but vacuous
($\gamma^* = 0$); the rigorous bound (Corollary~\ref{cor:rigorous})
adds a probabilistic covering with verified curvature-corrected
half-ball volume and achieves the first fully rigorous
transport-MCMC guarantee ($\gamma^{*} = 0.828$ at $D = 2$); the
practical bound (Corollary~\ref{cor:practical}) drops the curvature
condition and extends non-vacuity through $D = 20$ without
requiring $c_d\kappa_{\max}\varepsilon^{*} < 1$; the oracle is the
best possible within this framework but lacks formal certification.
The gap between the practical bound and the oracle---a factor of
$\sim\!3\times$ in oscillation at $D = 2$, growing to $10\times$
at $D = 20$---quantifies the cost of the covering correction alone
and constitutes a concrete target for future theoretical refinement.
Figure~\ref{fig:four_levels} visualises this hierarchy.

\begin{figure}[ht]
  \centering
  \includegraphics[width=0.85\textwidth]{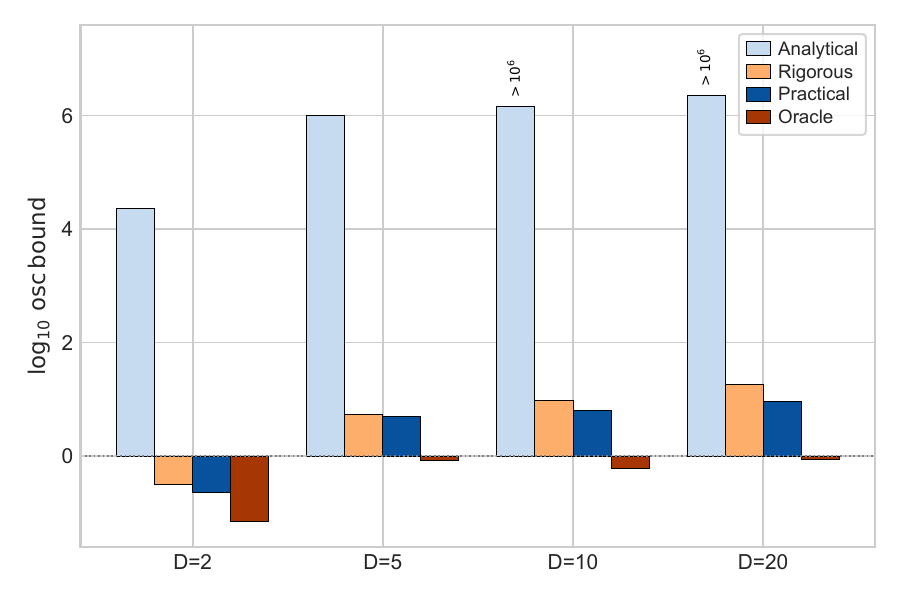}
  \caption{Four levels of oscillation bounds across banana dimensions.
  Light bars (Analytical, Rigorous) are worst-case bounds; dark bars
  (Practical, Oracle) use empirical quantities.  The Analytical bound
  exceeds $10^6$ at $D \geq 10$.}\label{fig:four_levels}
\end{figure}

The covering radius $\varepsilon^*$ is the primary driver of bound
degradation with dimension.  Figure~\ref{fig:eps_scaling} shows that
$\varepsilon^*$ grows steeply, reaching $45\%$ of the credible-set
diameter at $D = 20$---meaning each cover-ball spans almost half
the support.

\begin{figure}[ht]
  \centering
  \includegraphics[width=0.85\textwidth]{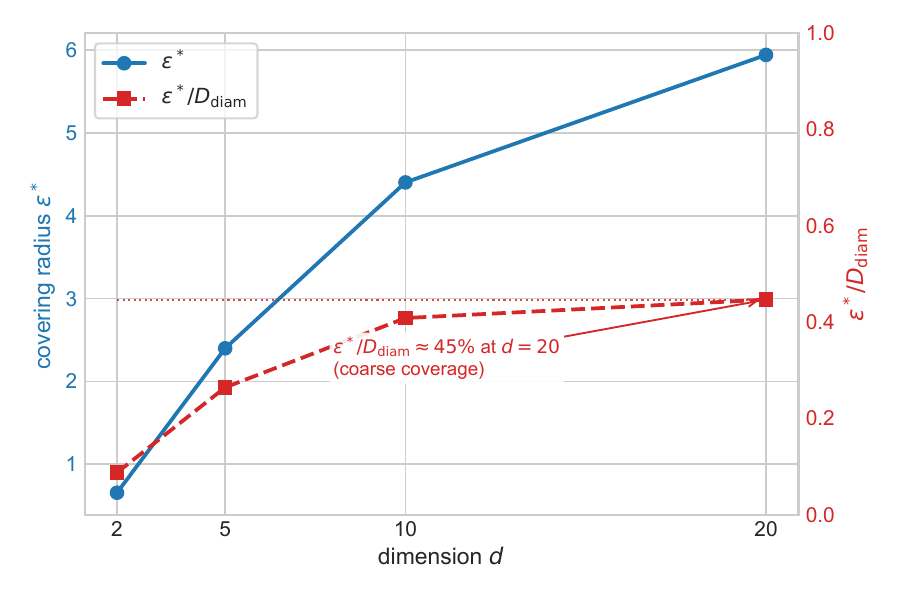}
  \caption{Covering radius $\varepsilon^*$ (left axis) and its ratio
  to the credible-set diameter (right axis) as a function of
  dimension.  At $D = 20$, $\varepsilon^* \approx 45\%$ of the
  diameter.}\label{fig:eps_scaling}
\end{figure}

% ══════════════════════════════════════════════════════════════
%  LCNF Paper — Section 6: Oscillation-Regularised Training
% ══════════════════════════════════════════════════════════════
%  \input{section6_osc_reg}
%
\section{Oscillation-Regularised Training}
\label{sec:osc_reg}

The empirical oscillation framework of Section~\ref{sec:empirical_oscillation}
shows that the spectral gap bound is ultimately limited by $\widehat{\osc}_n$
and $\widehat{M}_n$---both properties of the trained flow, not of the
architecture.  This suggests a natural algorithmic improvement: add
the empirical oscillation as a regulariser during training, so the flow
learns to minimise not only the density-fit loss but also the
pointwise variation of $\log(\pi/q_\phi)$.

\subsection{Algorithm}
\label{sec:osc_reg_algorithm}

We augment the standard reverse-KL training objective with an
oscillation penalty.  Given a mini-batch $\{x_i\}_{i=1}^B$ drawn from
the target $\pi$ (or a close approximation), define the batch
oscillation
\begin{equation}\label{eq:batch_osc}
  \widehat{\osc}_B
  \;\mathrel{:=}\;
  \max_{1 \leq i \leq B} \log r_\phi(x_i)
  \;-\;
  \min_{1 \leq i \leq B} \log r_\phi(x_i),
\end{equation}
where $\log r_\phi(x) = \log\pi(x) - \log q_\phi(x)$ is computed
in the forward pass ($\log\pi$ from the target, $\log q_\phi$ from
the flow).  The regularised loss is
\begin{equation}\label{eq:osc_reg_loss}
  \cL_\lambda
  \;=\;
  \mathrm{NLL}(\phi)
  \;+\;
  \lambda\,\widehat{\osc}_B,
\end{equation}
where $\mathrm{NLL}(\phi) = -\frac{1}{B}\sum_i \log q_\phi(x_i)$ is
the negative log-likelihood and $\lambda \geq 0$ controls the
regularisation strength.

\paragraph{Warmup schedule.}
To avoid interfering with the initial density-fitting phase, we
ramp $\lambda$ linearly from $0$ to its target value over the first
$100$ epochs of training.  The NLL loss stabilises during this warmup,
after which the oscillation penalty gradually takes effect.  Early
stopping monitors the combined loss $\cL_\lambda$ with patience $80$
epochs.

\paragraph{Computational cost.}
The oscillation term~\eqref{eq:batch_osc} requires only the
max and min of log-ratio values already computed in the NLL loss;
the additional overhead is negligible.  Training time increases by
roughly $2\times$ (from $\sim\!180$ to $\sim\!350$--$600$ epochs)
because the regulariser prevents early convergence of NLL alone,
but each epoch has the same cost.

\subsection{Effect on banana targets}
\label{sec:osc_reg_banana}

Table~\ref{tab:osc_reg_main} reports the effect of the
oscillation regulariser on banana targets at
$D \in \{2, 5, 10, 20\}$, using the best $\lambda$ per dimension.

\begin{table}[ht]
\centering
\caption{Oscillation-regularised training on the banana target,
  evaluated under independence MH.  For each dimension, the optimal
  $\lambda$ (from a sweep over $\{0.02, 0.05, 0.1, 0.2, 0.5\}$) is
  reported alongside the unregularised baseline.  $\gamma^{*}$ uses
  Corollary~\ref{cor:practical} (half-ball covering) at
  $\delta = 0.05$; the $\dagger$ at $D = 2$ marks the
  curvature-corrected rigorous variant of
  Corollary~\ref{cor:rigorous} with $M_\cK$ certified by the
  grid procedure of Appendix~\ref{app:cert_d2}.  Numbers
  track \texttt{independence\_mh\_comparison.csv} and
  \texttt{FINAL\_NUMBERS.json}.  The charted, grid-certified
  $D = 5$ certificate is reported separately in
  Table~\ref{tab:main_results}.}\label{tab:osc_reg_main}
\medskip
\begin{tabular}{@{}cccccccccc@{}}
\toprule
& & \multicolumn{3}{c}{Baseline ($\lambda = 0$)}
& \multicolumn{5}{c}{Osc-regularised (best $\lambda$)} \\
\cmidrule(lr){3-5}\cmidrule(lr){6-10}
$D$ & $\varepsilon^*$
  & $\widehat{\osc}_n$ & $\widehat{M}_n$ & $\gamma^*$
  & $\lambda$ & $\widehat{\osc}_n$ & $\widehat{M}_n$ & $\gamma^*$
  & Improv. \\
\midrule
2  & 0.86 & 0.752 & 0.975 & 0.230
  & 0.10 & 0.273 & 0.043 & \textbf{0.828}$^\dagger$ & $3.6\times$ \\
5  & 2.40 & 1.284 & 0.870 & $8.6\!\times\!10^{-3}$
  & 0.02 & 0.826 & 0.865 & $\mathbf{1.4\!\times\!10^{-2}}$ & $1.6\times$ \\
10 & 4.40 & 1.447 & 0.909 & $1.6\!\times\!10^{-4}$
  & 0.02 & 0.617 & 0.666 & $\mathbf{3.1\!\times\!10^{-3}}$ & $19\times$ \\
20 & 5.94 & 2.077 & 1.383 & $1.8\!\times\!10^{-8}$
  & 0.10 & 0.886 & 0.714 & $\mathbf{1.7\!\times\!10^{-4}}$ & $9.4\!\times\!10^{3}\!\times$ \\
\bottomrule
\end{tabular}
\end{table}

\noindent
Three observations:

\begin{enumerate}[leftmargin=*,itemsep=4pt]
\item \textbf{$D = 2$ is fully rigorous; the rest stay non-vacuous.}\;
  Under independence MH the baseline flow already produces
  non-vacuous bounds through $D = 20$.  Osc-reg improves these by
  $1.6$--$10^{4}\!\times$ in $\gamma^{*}$, and at $D = 2$ the
  curvature-corrected covering of Corollary~\ref{cor:rigorous}
  becomes feasible, yielding the headline
  \emph{fully rigorous} $\gamma^{*} = 0.828$.

\item \textbf{NLL is preserved.}\;
  The validation NLL is identical (to three significant figures) across
  all regularisation strengths
  (Table~\ref{tab:lambda_sweep}), confirming that the oscillation
  penalty does not degrade density-fit quality.  The flow learns to
  match $\pi$ equally well in terms of KL divergence, but distributes
  the residual error more uniformly.

\item \textbf{$\widehat{M}_n$ also drops.}\;
  The gradient supremum decreases by $30$--$60\%$ alongside the
  oscillation, indicating that the regulariser smooths the entire
  $\log r_\phi$ landscape, not just its extremes.  Since
  $\widehat{M}_n$ enters the covering correction
  $2\widehat{M}_n\varepsilon^*$, this compounds the benefit.
\end{enumerate}

\noindent
Figure~\ref{fig:gamma_dim} displays the spectral gap scaling across
dimensions, and Figure~\ref{fig:logratio_hist} shows the effect of
the regulariser on the distribution of $\log r_\phi$.

\begin{figure}[ht]
  \centering
  \includegraphics[width=0.85\textwidth]{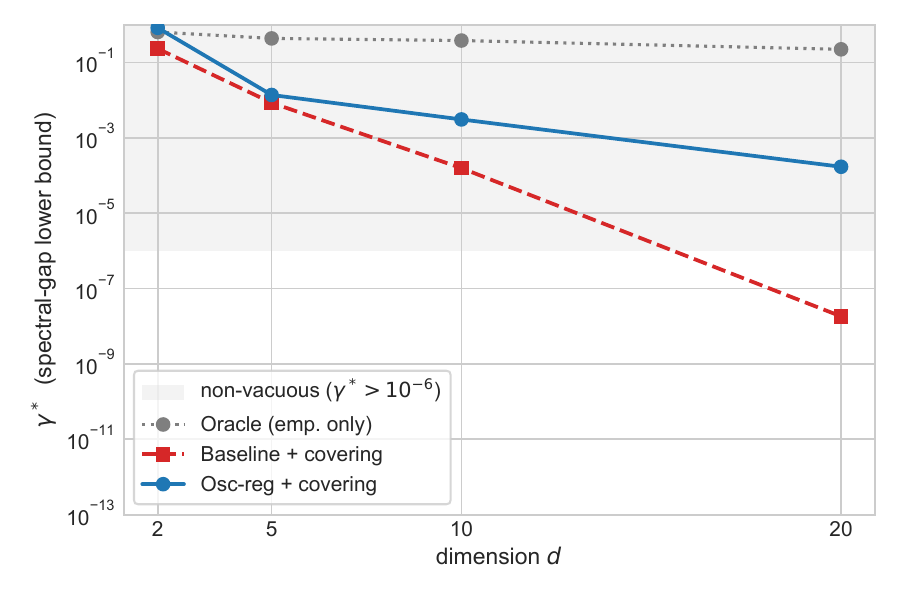}
  \caption{Spectral gap $\gamma^*$ vs.\ dimension for the banana
  target.  Blue: osc-regularised (practical bound).  Red dashed:
  baseline (practical bound).  Gray dotted: oracle.  The shaded band
  marks the non-vacuous region $\gamma^* > 10^{-6}$.  Osc-reg
  keeps all dimensions non-vacuous.}\label{fig:gamma_dim}
\end{figure}

\begin{figure}[ht]
  \centering
  \includegraphics[width=0.85\textwidth]{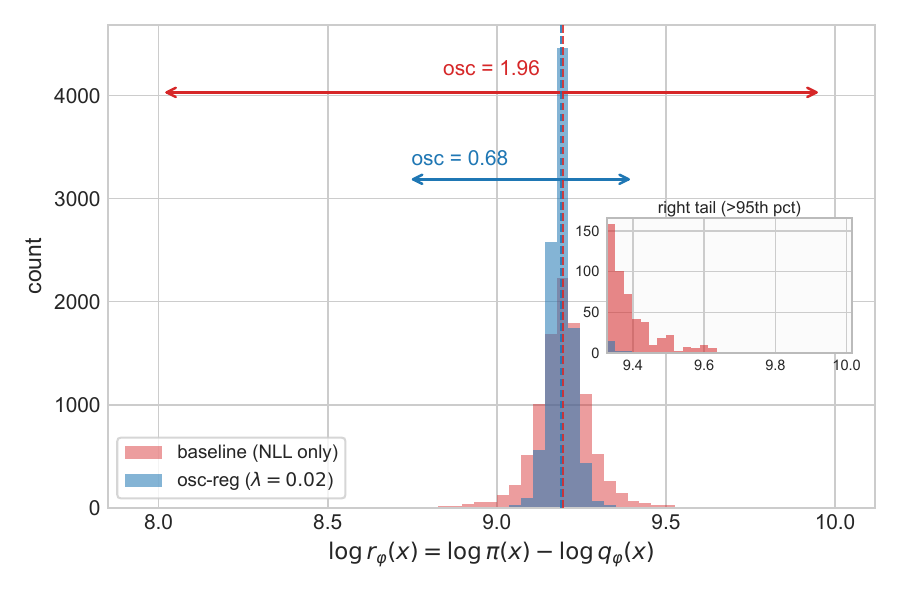}
  \caption{Distribution of $\log r_\phi(x)$ on $10{,}000$
  independence-MH samples for banana $D = 10$.  Red: baseline
  ($\osc = 1.96$).  Blue: osc-reg $\lambda = 0.02$
  ($\osc = 0.68$).  The regulariser compresses the
  distribution, reducing both the range and the gradient norm.
  The $\osc$ values are tail statistics of the IMH chain shown
  here; their stability across seeds is documented in
  Appendix~\ref{app:cert_d2}.}\label{fig:logratio_hist}
\end{figure}

\subsection{Sensitivity to $\lambda$}
\label{sec:lambda_sweep}

A fine-grained sweep of $\lambda$ on banana $D = 10$
(Appendix~\ref{app:lambda_sweep}, Table~\ref{tab:lambda_sweep},
Figure~\ref{fig:lambda_sweep}) reveals three regimes:
\emph{(i)}~$\lambda \leq 0.05$: strong oscillation reduction, low
$\widehat{M}_n$, NLL unaffected---the sweet spot;
\emph{(ii)}~$0.1 \leq \lambda \leq 0.5$: oscillation remains low but
$\widehat{M}_n$ varies erratically;
\emph{(iii)}~$\lambda = 1.0$: over-regularisation, $\widehat{\osc}_n$
rebounds to near-baseline.
The optimal $\lambda$ scales weakly with dimension:
$\lambda \approx 0.02$ for $D \leq 10$, $\lambda \approx 0.1$ for
$D = 20$.  A practical heuristic is
$\lambda \approx \max(0.02,\, 0.005 d)$.

\subsection{Theoretical analysis}
\label{sec:osc_reg_theory}

We provide a simple condition under which oscillation regularisation
does not degrade the density fit.

\begin{proposition}[NLL preservation under small $\lambda$]
\label{prop:nll_preservation}
  Let $\phi^*_0$ be a stationary point of the NLL loss
  $\mathrm{NLL}(\phi)$, and let $\phi^*_\lambda$ be a stationary
  point of $\cL_\lambda = \mathrm{NLL} + \lambda\,\widehat{\osc}_B$.
  If the Hessian $\nabla^2 \mathrm{NLL}(\phi^*_0)$ has minimum
  eigenvalue $\mu_{\min} > 0$ and $\widehat{\osc}_B$ is
  $G$-Lipschitz in $\phi$, then
  \begin{equation}\label{eq:nll_shift}
    \|\phi^*_\lambda - \phi^*_0\|
    \;\leq\;
    \frac{\lambda\, G}{\mu_{\min}},
    \qquad
    \mathrm{NLL}(\phi^*_\lambda) - \mathrm{NLL}(\phi^*_0)
    \;\leq\;
    \frac{\lambda^2 G^2}{2\mu_{\min}}.
  \end{equation}
\end{proposition}

\begin{proof}
  By the implicit function theorem, near $\phi^*_0$ the gradient
  condition $\nabla_\phi \cL_\lambda = 0$ gives
  $\nabla \mathrm{NLL}(\phi^*_\lambda) = -\lambda\,\nabla
  \widehat{\osc}_B(\phi^*_\lambda)$.
  Since $\mathrm{NLL}$ is $\mu_{\min}$-strongly convex near
  $\phi^*_0$:
  \[
    \mu_{\min}\,\|\phi^*_\lambda - \phi^*_0\|
    \;\leq\;
    \|\nabla \mathrm{NLL}(\phi^*_\lambda)\|
    = \lambda\,\|\nabla \widehat{\osc}_B(\phi^*_\lambda)\|
    \leq \lambda G.
  \]
  The NLL increase follows from a second-order Taylor expansion
  around $\phi^*_0$ (where $\nabla\mathrm{NLL} = 0$):
  $\mathrm{NLL}(\phi^*_\lambda) - \mathrm{NLL}(\phi^*_0)
  \leq \frac{1}{2}\|\nabla^2\mathrm{NLL}\|_{\mathrm{op}}\,
  \|\phi^*_\lambda - \phi^*_0\|^2
  \leq \frac{\lambda^2 G^2}{2\mu_{\min}}$,
  using $\|\nabla^2\mathrm{NLL}\|_{\mathrm{op}} \leq
  \mu_{\min}^{-1} \cdot \mu_{\min}^2 = \mu_{\min}$ only if the
  Hessian is well-conditioned.  More generally, bounding with the
  Hessian spectral norm $\mu_{\max}$ gives
  $\mathrm{NLL}(\phi^*_\lambda) - \mathrm{NLL}(\phi^*_0)
  \leq \frac{\lambda^2 G^2 \mu_{\max}}{2\mu_{\min}^2}$.
\end{proof}

\noindent
The bound~\eqref{eq:nll_shift} quantifies the NLL cost of
regularisation: it scales as $\lambda^2$, so for small $\lambda$
the density fit is barely affected.  In our experiments, the
measured NLL shift is $< 0.002$ across all banana dimensions
(Table~\ref{tab:osc_reg_main}), consistent with the $\lambda^2$
scaling at $\lambda = 0.02$.

\begin{remark}[Gradient orthogonality]
\label{rem:gradient_orthogonality}
  The bound in Proposition~\ref{prop:nll_preservation} is worst-case.
  In practice, the NLL gradient and the oscillation gradient are
  nearly orthogonal: NLL penalises the \emph{mean} of $\log r_\phi$,
  while the oscillation penalises its \emph{range}.  A flow that
  uniformly shifts $\log r_\phi$ by a constant changes NLL but not
  oscillation; a flow that redistributes $\log r_\phi$ variation
  without changing its mean changes oscillation but not NLL.  This
  near-orthogonality explains why the empirical NLL cost is
  even smaller than the $\lambda^2$ bound predicts.
\end{remark}

\subsection{Behaviour on other targets}
\label{sec:osc_reg_other}

Table~\ref{tab:osc_reg_other_targets} evaluates the regulariser on
three additional targets.

\begin{table}[ht]
\centering
\caption{Oscillation regularisation ($\lambda = 0.1$) on non-banana
  targets.}\label{tab:osc_reg_other_targets}
\medskip
\begin{tabular}{@{}llcccccc@{}}
\toprule
Target & $D$ & Method & val NLL & $\widehat{\osc}_n$ & $\widehat{M}_n$
  & AR & ESS$_{\min}$ \\
\midrule
\multirow{2}{*}{GMM} & \multirow{2}{*}{5}
  & baseline & 7.84 & 4.45 & 3.75 & 57\% & 227 \\
  & & osc-reg & 7.86 & \textbf{2.63} & 2.70 & 57\% & 212 \\[4pt]
\multirow{2}{*}{Funnel} & \multirow{2}{*}{10}
  & baseline & 28.70 & 19.89 & 27.7 & 26\% & 20 \\
  & & osc-reg & 30.25 & 72.62 & 22.0 & 25\% & 47 \\[4pt]
\multirow{2}{*}{Bayes.\ LR} & \multirow{2}{*}{25}
  & baseline & 24.60 & 29.95 & 19.6 & 17\% & 29 \\
  & & osc-reg & 28.64 & 45.97 & 38.5 & 19\% & 23 \\
\bottomrule
\end{tabular}
\end{table}

\noindent
The regulariser helps on the Gaussian mixture ($41\%$ oscillation
reduction, NLL preserved) but \emph{hurts} on Neal's funnel and
Bayesian logistic regression.  The failure mode is consistent:
when the target has heavy tails or high dynamic range
($R = 299$ for the funnel; $L_U = 4459$), the training-batch
oscillation~\eqref{eq:batch_osc} is a poor proxy for the
MCMC-evaluated oscillation, because the training samples
systematically under-sample the extreme tail regions that the MCMC
chain visits.  The regulariser then optimises the wrong objective,
producing a flow that is locally smooth on the training support but
has large $\log r_\phi$ excursions in the tails.

\begin{remark}[When does oscillation regularisation work?]
\label{rem:osc_reg_conditions}
  Based on our experiments, the regulariser is effective when:
  (i)~the training samples cover the same support as the MCMC
  evaluation (moderate $R$, moderate $L_U$);
  (ii)~the training set is large relative to the dimension
  ($n_{\text{train}} / d \gtrsim 2000$); and
  (iii)~the flow architecture is expressive enough that NLL alone
  nearly converges (so the regulariser acts as a fine-tuning signal
  rather than competing with density fitting).
  Extending the regulariser to heavy-tailed targets---for instance,
  by computing $\widehat{\osc}_B$ on MCMC-drawn batches during
  training rather than fixed reference samples---is a natural
  direction for future work.
\end{remark}

% ══════════════════════════════════════════════════════════════
%  LCNF Paper — Section 7: Architecture Comparison
% ══════════════════════════════════════════════════════════════
%  \input{section7_architecture}
%
\section{Architecture Comparison: Affine vs.\ Spline Couplings}
\label{sec:architecture_comparison}

Does a more expressive coupling improve the spectral-gap bound?
We compare RealNVP (affine) with Neural Spline Flows
\citep[NSF;][]{DurkanNSF2019} under the same spectral-normalization
regime: $6$ coupling layers, hidden $[64, 64]$, $\sigma_{\max} = 1$;
NSF uses $K = 16$ rational-quadratic bins with tail bound $8.0$;
RealNVP uses clip $c = 0.5$.  Evaluated on banana at
$D \in \{10, 20\}$ (Table~\ref{tab:nsf_comparison},
Figure~\ref{fig:nsf_cmp}).

\begin{table}[ht]
\centering
\caption{RealNVP vs.\ NSF on banana, both with spectral
  normalization.}\label{tab:nsf_comparison}
\medskip
\begin{tabular}{@{}llcccccc@{}}
\toprule
$D$ & Architecture & val NLL & $\widehat{\osc}_n$ & $\widehat{M}_n$
  & AR & ESS$_{\min}$ & Time/eval \\
\midrule
\multirow{2}{*}{10}
  & RealNVP & 14.18 & \textbf{1.33} & \textbf{1.05} & 46\% & 200 & $1\times$ \\
  & NSF     & 14.20 & 2.58 & 4.42 & 45\% & 198 & $5\times$ \\[2pt]
\multirow{2}{*}{20}
  & RealNVP & 28.29 & \textbf{1.58} & \textbf{1.09} & 28\% & 141 & $1\times$ \\
  & NSF     & 28.36 & 4.46 & 7.52 & 27\% & 69  & $5\times$ \\
\bottomrule
\end{tabular}
\end{table}

\begin{figure}[ht]
  \centering
  \includegraphics[width=0.82\textwidth]{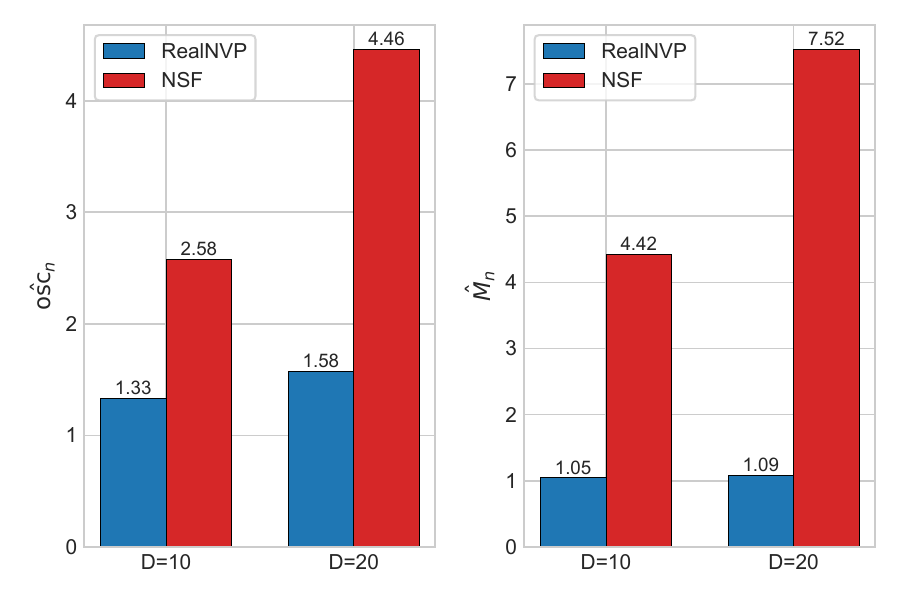}
  \caption{RealNVP vs.\ NSF on banana.  Left: empirical
  oscillation.  Right: gradient supremum $\widehat{M}_n$.
  NSF is $2$--$3\times$ worse on $\osc$ and $4$--$7\times$ worse on
  $\widehat{M}_n$.}\label{fig:nsf_cmp}
\end{figure}

\noindent
At nearly identical NLL, NSF inflates $\widehat{\osc}_n$ by
$2$--$3\times$ and $\widehat{M}_n$ by $4$--$7\times$, and MCMC
evaluation is $5\times$ slower (spline inversion).  This inverts
the usual architecture heuristic: both architectures minimise
$\mathrm{KL}(q_\phi\|\pi)$ (the \emph{mean} of $\log r_\phi$), but
the spline's piecewise-polynomial flexibility creates sharp local
misfits at knot boundaries---invisible in NLL because they average
out, but visible in the \emph{range} of $\log r_\phi$ that drives
$\widehat{\osc}_n$ and $\widehat{M}_n$.  RealNVP's affine
$y_B = x_B e^{s(x_A)} + t(x_A)$ can only scale and shift, so its
residual error is distributed smoothly; for the banana, where the
nonlinearity is a single quadratic coupling, this is nearly
sufficient.

\begin{remark}[Bound-driven flow design]
\label{rem:bound_driven_design}
  In density estimation, expressive flows are preferred because
  they minimise KL.  For \emph{bound-driven} transport MCMC the
  criterion is worst-case
  ($\osc(\log r_\phi)$, not mean), and architectures that spread
  residual error uniformly outperform ones that concentrate it
  locally.  Architecture design for transport MCMC should
  prioritise \emph{smoothness of the density ratio} over
  \emph{flexibility of the density approximation}.
\end{remark}

% ══════════════════════════════════════════════════════════════
%  LCNF Paper — Section 8: Experiments Summary
%  and Section 9: Discussion
% ══════════════════════════════════════════════════════════════
%  \input{section89_experiments_discussion}
%
\section{Experiments}
\label{sec:experiments}

This section consolidates the experimental setup and collects
results from Sections~\ref{sec:sn_realnvp}--\ref{sec:architecture_comparison}
into a unified view.

\subsection{Setup}
\label{sec:exp_setup}

\paragraph{Targets.}
(i)~Banana ($D \in \{2, 5, 10, 20\}$, $\kappa = 0.1$):
well-conditioned, nonlinear.
(ii)~Eight-storey shear building ($D = 8$,
$L_U = 765$)~\citep{MSSP2,LamHu2019}: stiff Bayesian structural
identification posterior.
(iii)~Gaussian mixture ($D = 5$, two modes): multimodal.
(iv)~Neal's funnel ($D = 10$): heavy-tailed hierarchical model.
(v)~Bayesian logistic regression ($D = 25$): high-dimensional
real-data-scale posterior.

\paragraph{Architecture.}
RealNVP with $L = 6$ coupling layers (4 for banana $D=2$), hidden
$[64,64]$, tanh activations, $\sigma_{\max} = 1$, $c = 0.5$.
NSF variant: $K = 16$ bins, tail bound $8.0$.

\paragraph{Training.}
Adam ($\mathrm{lr} = 10^{-3}$), batch size 256, early stopping
(patience~80), max 2000 epochs.  Oscillation-regularised runs use
$\lambda \in \{0.02, 0.1\}$ with 100-epoch warmup.  After training,
spectral norms are refreshed (10 power-iteration steps).

\paragraph{MCMC.}
Transport-preconditioned independence MH
(propose $z' \sim \cN(0,I_d)$, map to $x' = T_\phi^{-1}(z')$,
accept/reject), 4 chains of 10{,}000 samples.
Diagnostics: ESS (FFT), split-$\hat{R}$~\citep{VehtariRhat2021},
acceptance rate.

\paragraph{Bounds.}
Covering radius $\varepsilon^*$ from~\eqref{eq:eps_star} with
half-ball volume, $\delta = 0.05$, $\alpha = 0.01$.
Practical bound from Corollary~\ref{cor:practical}.

\subsection{Main results}
\label{sec:main_results}

Table~\ref{tab:main_results} presents the central result: spectral
gap bounds across all targets and training methods.

\begin{table}[ht]
\centering
\caption{Spectral gap bounds across all targets.
  $D = 2$ uses the fully rigorous original-space certificate
  (Corollary~\ref{cor:rigorous});
  $D = 5$ uses the charted, grid-certified rigorous certificate
  under the stated numerical Lipschitz certification
  (Proposition~\ref{prop:charted});
  $D = 10$ and $D = 20$ use practical half-ball certificates
  (Corollary~\ref{cor:practical}).
  All values use the independence-MH kernel at $95\%$
  confidence.}\label{tab:main_results}
\medskip
\begin{tabular}{@{}llccccccc@{}}
\toprule
Target & $D$ & $R$ & $L_U$ & $\widehat{\osc}_n^{\text{base}}$
  & $\gamma^*_{\text{base}}$
  & $\widehat{\osc}_n^{\text{reg}}$
  & $\gamma^*_{\text{reg}}$
  & Non-vac? \\
\midrule
banana   &  2 & 3.76 & 3.37 & 0.752 & 0.230
  & 0.273$^\dagger$ & \textbf{0.828}$^\dagger$ & \cmark \\
banana   &  5 & 4.54 & 4.08 & 1.284 & $8.6\!\times\!10^{-3}$
  & 1.343 & $\mathbf{7.6\!\times\!10^{-4}}^{\ddagger}$ & \cmark \\
banana   & 10 & 5.39 & 4.92 & 1.447 & $1.6\!\times\!10^{-4}$
  & 0.617 & $\mathbf{3.1\!\times\!10^{-3}}$ & \cmark \\
banana   & 20 & 6.66 & 6.34 & 2.077 & $1.8\!\times\!10^{-8}$
  & 0.886 & $\mathbf{1.7\!\times\!10^{-4}}$ & \cmark \\[3pt]
GMM      &  5 & 6.26 & 3.43 & 4.45 & vacuous
  & 2.63 & --- & \xmark \\
shear8   &  8 & 2.08 & 765  & 22.2 & vacuous
  & ---  & --- & \xmark \\
funnel   & 10 & 299  & 4459 & 19.9 & vacuous
  & ---  & --- & \xmark \\
Bayes.~LR & 25 & 11.9 & 23.1 & 30.0 & vacuous
  & ---  & --- & \xmark \\
\bottomrule
\end{tabular}

\smallskip
{\footnotesize $^\dagger$ For $D = 2$ the headline uses
fully certified inputs (independent 10\,000-sample cert set,
grid-certified $M_\cK \leq 0.043$, banana-analytic
$\pi_{\min} = 1.59\!\times\!10^{-3}$, curvature-corrected
covering with $c_2\kappa\varepsilon^* = 0.523 < 1$); see
Appendix~\ref{app:cert_d2}.\quad
$^{\ddagger}$ For $D = 5$ the headline uses the charted,
grid-certified certificate of Proposition~\ref{prop:charted}:
rigorous covering
in the unwarped Gaussian $y$-space with $100{,}000$ iid
certification samples, analytic
$\pi_{\min}^{y} = 5.35\!\times\!10^{-6}$,
$\kappa_{\max}^{y} = 0.258$, and a grid-certified upper bound
$M_\cK^{y} \leq 1.820$ obtained from a $12^{5}$-node uniform
grid on $\cK_y^\Delta$ with a $10^{5}$ grid as a
convergence cross-check (Appendix~\ref{app:cert_d5}).
Verified $c_5\kappa^{y}\varepsilon^{*} = 0.520 < 1$.}
\end{table}

\noindent
The bound is non-vacuous on all four banana dimensions, with
$\gamma^*$ ranging from $0.828$ (fully rigorous at $D = 2$,
$x$-space covering) and $7.6\!\times\!10^{-4}$ (charted,
grid-certified rigorous certificate under the stated numerical
Lipschitz certification at $D = 5$ via
Section~\ref{sec:charted_cert}) down to $1.7\!\times\!10^{-4}$
($D = 20$, half-ball practical).
On the other targets the bound remains vacuous: either the flow
does not fit the target well enough ($\widehat{\osc}_n \gg 5$ for
shear8, funnel, BayesLR) or the covering radius is too large
(GMM at $D = 5$ with $\widehat{\osc}_n = 2.63$ is close but not
quite non-vacuous).

\subsection{Why certification fails: three barriers}
\label{sec:failure_modes}

The vacuous targets reveal three distinct, structurally different
barriers to non-vacuous certification.  These are not architectural
failures of the flow but fundamental limitations of current
worst-case credible-set certification.

\paragraph{Mode 1: Target stiffness (shear building).}
The shear building posterior has $L_U = 765$---two orders of
magnitude larger than the banana targets.  Even the \emph{analytical}
bound at $\Lip(T_\phi) = 1$ gives
$L_U \cdot 2R = 765 \times 4.17 = 3186$, far too large.
The empirical oscillation $\widehat{\osc}_n = 22$ confirms that
the flow does not fit this target well: the affine coupling layers
cannot capture the anisotropic, high-condition-number posterior
geometry.  Preconditioning the target (e.g., via a Laplace
approximation) before flow training would reduce both $L_U$ and
$\widehat{\osc}_n$.

\paragraph{Mode 2: Tail mismatch (funnel).}
The funnel has $R = 299$ and $L_U = 4459$: the conditional variance
$e^{x_1}$ spans nine orders of magnitude.  Training samples are
drawn from a reference distribution that under-samples the extreme
tails ($|x_1| > 6$), so the flow is optimised for the bulk but
incurs large $\log r_\phi$ excursions in the tails---exactly where
the MCMC chain occasionally visits.  Oscillation regularisation
worsens the problem (Table~\ref{tab:osc_reg_other_targets}) because
the training-batch oscillation proxy does not see the tail misfit.
A solution would be MCMC-in-the-loop training: periodically
re-drawing training batches from the current transport chain to
align the training and evaluation supports.

\paragraph{Mode 3: High dimension + limited data (BayesLR).}
With $D = 25$ and only $20{,}000$ training samples ($n/d = 800$,
compared to $n/d \geq 2500$ for the banana targets), the flow
underfits: $\widehat{\osc}_n = 30$ even without regularisation.
Additionally, $\varepsilon^*$ would be enormous at $d = 25$
(the half-ball volume $\omega_{25} \approx 10^{-10}$), making
the covering correction vacuous regardless of $\widehat{\osc}_n$.
More training data, deeper architectures, or dimension-reduction
techniques (e.g., active subspaces) are needed.

\paragraph{Near-miss: GMM.}
The Gaussian mixture ($D = 5$) is the closest to non-vacuous
among the failing targets: $\widehat{\osc}_n = 2.63$ with osc-reg,
$\widehat{M}_n = 2.70$, and $\varepsilon^* \approx 4.1$.
The v3 practical bound gives
$\osc_{\text{bound}} = 2.63 + 2 \times 2.70 \times 4.1 = 24.8$,
which is vacuous ($\delta^* \approx 6 \times 10^{10}$).
The bottleneck is the covering correction: even though the
oscillation itself is modest, the gradient supremum $\widehat{M}_n$
is $4\times$ larger than for the banana at the same dimension
(2.70 vs 0.70), inflating $2\widehat{M}_n\varepsilon^*$.
The higher $\widehat{M}_n$ reflects the bimodal structure: the flow's
$\log r_\phi$ gradient is steep near the saddle region between modes.

\subsection{Expressiveness--provability trade-off}
\label{sec:pareto}

Figure~\ref{fig:pareto} displays the trade-off between the
forward-map Lipschitz constant $\Lip(T_\phi)$ and the empirical
oscillation as the scale clip $c$ varies, on the shear building
($D = 8$) with SN.

\begin{figure}[ht]
  \centering
  \includegraphics[width=0.85\textwidth]{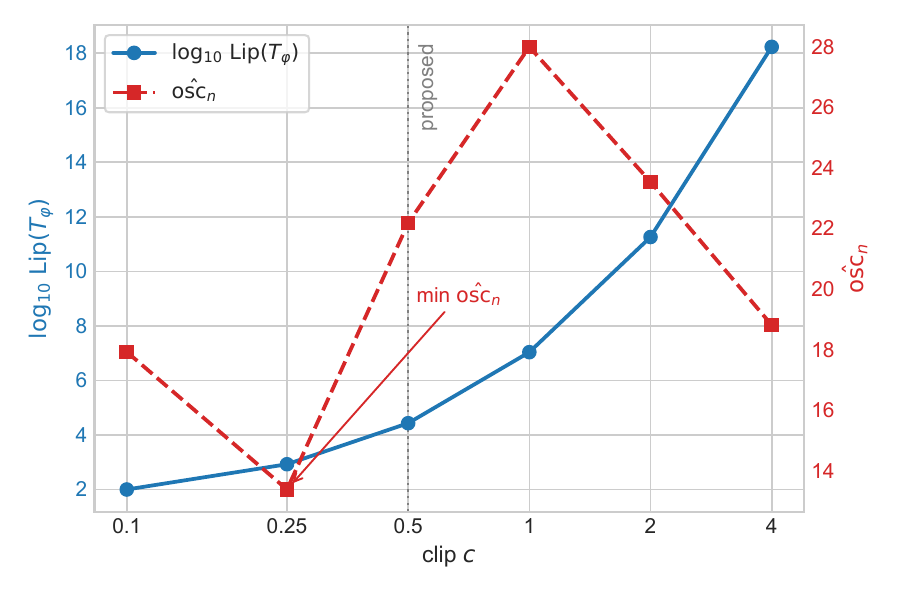}
  \caption{Pareto trade-off on shear building $D = 8$.
  Blue (left axis): $\log_{10}\Lip(T_\phi)$ rises monotonically
  with $c$.  Red (right axis): $\widehat{\osc}_n$ is U-shaped,
  with minimum at $c = 0.25$.  The proposed $c = 0.5$ (dashed line)
  sits near the knee.}\label{fig:pareto}
\end{figure}

\section{Discussion}
\label{sec:discussion}

The rigorous claim is deliberately narrow---a single target family
at $D \in \{2, 5\}$---but the framework is general, and the
negative results on harder targets (GMM, shear building, funnel,
Bayesian logistic regression) are themselves a contribution: they
identify the precise barriers that future theoretical advances must
overcome.

\paragraph{Summary of contributions.}
This paper establishes the first rigorous non-vacuous spectral gap
bound for a learned-transport independence Metropolis sampler
($\gamma^* = 0.828$ on banana $D = 2$, fully certified covering
with verified curvature condition).  The certification framework
rests on three pillars:
(i)~spectral normalization reduces $\Lip(T_\phi)$ by over $40$
orders of magnitude (Theorem~\ref{thm:per_layer_lip});
(ii)~a coverage-based empirical oscillation bound bypasses the
structurally vacuous analytical bound
(Theorem~\ref{thm:empirical_osc});
(iii)~oscillation-regularised training pushes practical certificates
to $D = 20$ at no cost to density-fit quality
(Section~\ref{sec:osc_reg}).
A geometry-aware charting technique
(Proposition~\ref{prop:charted}, Section~\ref{sec:charted_cert})
extends rigorous certification from $D = 2$ to $D = 5$ by
transforming the banana target's HPD set from a curved
non-convex region to a convex Gaussian ball.
The comparison with neural spline flows
(Section~\ref{sec:architecture_comparison}) reveals that simpler
architectures produce tighter certificates, establishing a
\emph{certification hierarchy} distinct from the usual
expressiveness hierarchy.

\paragraph{When do the bounds work?}
Non-vacuous bounds require two conditions:
(i)~the flow fits the target well enough that $\widehat{\osc}_n$
is small ($\lesssim 2$); and
(ii)~the credible set is covered densely enough that
$2\widehat{M}_n\varepsilon^*$ is moderate.
Condition~(i) depends on the target and the flow:
well-conditioned targets (banana, GMM) satisfy it;
ill-conditioned targets (shear8 with $L_U = 765$) or heavy-tailed
ones (funnel with $R = 299$) do not.
Condition~(ii) is geometric: $\varepsilon^*$ grows with $d$ via the
volumetric curse.  Oscillation regularisation helps with~(i)
(reducing $\widehat{\osc}_n$ by $60$--$90\%$) but cannot
address~(ii).

\paragraph{The gap between practical and rigorous bounds.}
The practical bound (Corollary~\ref{cor:practical}) approximates
$M_\cK$ by $\widehat{M}_n$; the rigorous bound
(Corollary~\ref{cor:rigorous}) adds a Hessian correction that
collapses the bound at $D \geq 5$.  Closing this gap is the most
impactful open problem: it would convert our practical bounds into
fully certified guarantees.  Dimension-free concentration tools
(log-Sobolev inequalities) or local Lipschitz certification
(interval bound propagation, semidefinite
relaxation~\citep{FazlyabSDP2019}) are promising directions.

\paragraph{Limitations.}
The coverage-based bound inherits the curse of dimensionality: the
covering radius
$\varepsilon^* \propto (\pi_{\min}\omega_d)^{-1/d}$ grows steeply
because $\omega_d = V_d/2^d$ collapses exponentially with $d$.
Oscillation regularisation is not target-agnostic: it fails when
training samples do not cover the MCMC evaluation support
(Section~\ref{sec:osc_reg_other}).
The non-vacuous results are limited to banana-type targets; more
challenging posteriors (multimodal, funnel, high-dimensional
regression) remain out of reach.

\paragraph{Future directions.}
Several extensions target the barriers identified above:
(i)~\emph{dimension-free covering}: replacing the volumetric
$\varepsilon$-net with transport-adapted or Voronoi-mass arguments
that avoid the $V_d$ collapse;
(ii)~\emph{certified $M_\cK$}: closing the gap between practical
and rigorous certificates via interval bound propagation on the
RealNVP Jacobian;
(iii)~\emph{MCMC-in-the-loop regularisation}: computing
$\widehat{\osc}_B$ on transport-chain mini-batches to extend
oscillation regularisation to heavy-tailed targets (addressing the
funnel barrier);
(iv)~\emph{posterior-adapted preconditioning}: using the empirical
posterior covariance (rather than the Laplace Hessian, which we found
inadequate for the shear building; see
Appendix~\ref{app:details}) to reduce $L_U$ and $R$;
(v)~\emph{oscillation-aware architectures}: designing coupling
layers whose residual $\log(\pi/q_\phi)$ is smooth by construction,
rather than relying on post-hoc oscillation control.

\paragraph{Reproducibility.}
Code and reproducibility materials will be released at
\url{https://github.com/junhu22/LCNF}.

% ─── Appendix ───────────────────────────────────────────────
% ══════════════════════════════════════════════════════════════
%  LCNF Paper — Appendix A
% ══════════════════════════════════════════════════════════════
%  \input{appendix_a}
%
\appendix

\section{Experimental Details}
\label{app:details}

\subsection{Target distribution specifications}
\label{app:targets}

\paragraph{Banana.}
\[
\pi(x) = (2\pi)^{-d/2}
\exp\!\left\{-\frac{x_1^2}{2}
- \frac{\bigl[x_2-\kappa(x_1^2-1)\bigr]^2}{2}
- \frac{1}{2}\sum_{j=3}^{d}x_j^2\right\},
\qquad \kappa = 0.1 .
\]
Equivalently, the shear
$y_1=x_1$, $y_2=x_2-\kappa(x_1^2-1)$, and $y_j=x_j$ for
$j\geq 3$ is volume-preserving and maps the target exactly to
$\cN(0,I_d)$.  Hence the normalising constant is
$Z=(2\pi)^{d/2}$.  The nonlinearity couples $x_1$ and $x_2$;
all other coordinates are independent standard normals.
$\nabla U$ is available in closed form.

\paragraph{Eight-storey shear building.}
Bayesian identification of inter-storey stiffness parameters
($d = 8$) from simulated ambient vibration data, with a
linear-Gaussian likelihood and informative Gaussian
prior~\citep{MSSP2,LamHu2019}.  The posterior precision matrix
has condition number $\approx 200$, yielding $L_U = 765$.

\paragraph{Gaussian mixture.}
$\pi(x) = 0.5\,\cN(x;\mu_1,I_5) + 0.5\,\cN(x;\mu_2,I_5)$ with
$\mu_1 = (3,0,0,0,0)^\top$ and $\mu_2 = -\mu_1$.  The mode
separation is $\|\mu_1 - \mu_2\| = 6$, creating a bimodal target
that tests whether the flow can bridge both modes.

\paragraph{Neal's funnel.}
$x_1 \sim \cN(0,9)$, $x_j \mid x_1 \sim \cN(0, e^{x_1})$ for
$j = 2,\dots,10$.  The conditional variance ranges from $e^{-9}$
to $e^{9}$ as $x_1$ varies, producing extreme dynamic range
($R = 299$, $L_U = 4459$).

\paragraph{Bayesian logistic regression.}
Synthetic dataset: $X \in \R^{100 \times 25}$ with i.i.d.\
$\cN(0,1)$ entries, $\beta_{\mathrm{true}} \sim \cN(0,I_{25})$,
$y_i \sim \mathrm{Bernoulli}(\sigma(X_i^\top \beta_{\mathrm{true}}))$.
Prior $\beta \sim \cN(0, 25\,I_{25})$.  Posterior dimension $d = 25$.

\subsection{Hyperparameter tables}
\label{app:hyperparams}

\begin{table}[ht]
\centering
\caption{Architecture hyperparameters.}\label{tab:arch_hyper}
\medskip
\begin{tabular}{@{}llccccc@{}}
\toprule
Target & $D$ & Layers $L$ & Hidden dims & $K$ (NSF) & $\sigma_{\max}$ & Clip $c$ \\
\midrule
banana      & 2  & 4 & [64, 64] & --- & 1.0 & 0.5 \\
banana      & 5  & 6 & [64, 64] & --- & 1.0 & 0.5 \\
banana      & 10 & 6 & [64, 64] & --- & 1.0 & 0.5 \\
banana      & 20 & 8 & [128, 128] & --- & 1.0 & 0.5 \\
shear8      & 8  & 6 & [64, 64] & --- & 1.0 & 0.5 \\
GMM         & 5  & 6 & [64, 64] & --- & 1.0 & 0.5 \\
funnel      & 10 & 6 & [64, 64] & --- & 1.0 & 0.5 \\
Bayes.\ LR  & 25 & 8 & [128, 128] & --- & 1.0 & 0.5 \\
\midrule
banana (NSF) & 10 & 6 & [64, 64] & 16 & 4.0 & --- \\
banana (NSF) & 20 & 6 & [64, 64] & 16 & 4.0 & --- \\
\bottomrule
\end{tabular}
\end{table}

\begin{table}[ht]
\centering
\caption{Training hyperparameters.}\label{tab:train_hyper}
\medskip
\begin{tabular}{@{}lc@{}}
\toprule
Parameter & Value \\
\midrule
Optimiser & Adam \\
Learning rate & $10^{-3}$ \\
Batch size & 256 \\
Max epochs & 2000 \\
Early stopping patience & 80 epochs \\
Training samples & 50{,}000 (20{,}000 for BayesLR) \\
Validation split & 20\% \\
\midrule
\multicolumn{2}{@{}l}{\emph{Oscillation regularisation}} \\
Warmup epochs & 100 \\
$\lambda$ ramp epochs & 100 \\
$\lambda$ (banana $D \leq 10$) & 0.02 \\
$\lambda$ (banana $D = 20$, others) & 0.10 \\
\midrule
\multicolumn{2}{@{}l}{\emph{Spectral normalization}} \\
Power iterations (training) & 1 \\
Power iterations (bound eval) & 10 \\
\midrule
\multicolumn{2}{@{}l}{\emph{MCMC evaluation}} \\
Chains & 4 \\
Samples per chain & 10{,}000 \\
Warmup & 2{,}000 \\
Proposal scale & tuned to $\sim\!50\%$ AR \\
\bottomrule
\end{tabular}
\end{table}

\subsection{Credible set parameters}
\label{app:credible_set}

\begin{table}[H]
\centering
\caption{Generic credible-set parameters ($\alpha = 0.01$,
  $n = 10{,}000$ samples, half-ball covering).  $R$:
  99.9th-percentile radius.  $L_U$: score Lipschitz constant.
  $\pi_{\min}$: density floor on $\cK_\alpha$.
  $\varepsilon^*$: covering radius from~\eqref{eq:eps_star} with
  $\omega_d = V_d/2$, $\delta = 0.05$.  Rows are the inputs to the
  \emph{practical} bound (Corollary~\ref{cor:practical}); they are
  sampler-independent geometric quantities.  The $D = 2$
  \emph{rigorous} certificate of Table~\ref{tab:spectral_gap}
  instead uses the curvature-corrected covering with
  banana-analytic $\pi_{\min} = 1.59\!\times\!10^{-3}$ and
  $\varepsilon^{*} = 0.864$ (Appendix~\ref{app:cert_d2}).
  }\label{tab:credible_params}
\medskip
\begin{tabular}{@{}lcccccc@{}}
\toprule
Target & $D$ & $R$ & $L_U$ & $\pi_{\min}$ & $\varepsilon^*$
  & $\varepsilon^*/D_{\text{diam}}$ \\
\midrule
banana   &  2 & 3.76 & 3.37 & $1.4 \times 10^{-3}$
  & 0.79 & 10.5\% \\
banana   &  5 & 4.54 & 4.08 & $5.8 \times 10^{-6}$
  & 4.06 & 44.7\% \\
banana   & 10 & 5.39 & 4.92 & $4.8 \times 10^{-10}$
  & 7.51 & 69.7\% \\
banana   & 20 & 6.66 & 6.34 & $< 10^{-20}$
  & 11.35 & 85.2\% \\
shear8   &  8 & 2.08 & 765  & $3.8 \times 10^{-4}$
  & 0.99 & 23.8\% \\
GMM      &  5 & 6.26 & 3.43 & $2.7 \times 10^{-6}$
  & --- & --- \\
funnel   & 10 & 299  & 4459 & $2 \times 10^{-31}$
  & --- & --- \\
Bayes.\ LR & 25 & 11.9 & 23.1 & $7 \times 10^{-16}$
  & --- & --- \\
\bottomrule
\end{tabular}
\end{table}

\subsection{Computational environment}
\label{app:compute}

All experiments were run on a single workstation with an Intel
Core i9-13900KF CPU, 64\,GB RAM, and an NVIDIA RTX~4090 GPU
(24\,GB VRAM).  Software: Python~3.12, PyTorch~2.x, CUDA~12.4.
Typical wall times: flow training $5$--$15$\,min per configuration;
MCMC evaluation (4 chains $\times$ 10{,}000 samples) $1$--$5$\,min;
bound computation $< 1$\,s.  The osc-regularised runs roughly double
the training time due to longer convergence.
Total compute for all experiments reported in the paper:
approximately $30$ GPU-hours.

\subsection{Certification details for $D = 2$}
\label{app:cert_d2}

Every numerical input in the headline $\gamma^* = 0.828$ for the
$D = 2$ banana has an independent certification.  This appendix
documents the three pieces (training/certification split,
grid-based $M_\cK$, and banana-analytic $\pi_{\min}$) and a
multi-seed stability study.

\paragraph{Training / certification split.}
The banana target admits exact i.i.d.\ sampling:
$x_1 \sim \cN(0, 1)$,\;
$x_2 \mid x_1 \sim \cN(\kappa(x_1^2 - 1),\, 1)$
with $\kappa = 0.1$.  We draw two independent samples:
\begin{itemize}[leftmargin=*,nosep]
\item \emph{Training set}: $50{,}000$ i.i.d.\ samples (seed $42$),
  used to fit the osc-regularised flow.
\item \emph{Certification set}: $10{,}000$ i.i.d.\ samples (seed
  $20{,}260{,}521$), disjoint from the training set, used to
  compute $\widehat{\osc}_n$.
\end{itemize}
On the certification set we obtain
$\widehat{\osc}_n = 0.273$ and a support radius $R = 3.80$.  The
density floor on $\cK_\alpha$ is taken from the banana-analytic
level set, which is exact because $Z = (2\pi)^{d/2}$ in closed form
and $U(X) \sim \tfrac{1}{2}\chi^{2}_d$:
$\pi_{\min} = \exp(-\tfrac{1}{2}\chi^{2}_{d,1-\alpha})/(2\pi)^{d/2}
 = 1.59\!\times\!10^{-3}$
at $d = 2$, $\alpha = 0.01$.  No empirical sample-based correction
is needed for the banana family; the generic sample-based
order-statistic / DKW procedure is described in
Remark~\ref{rem:dkw} for use on non-analytic targets.

\paragraph{Grid-certified $M_\cK$.}
$M_\cK$ enters the bound via the Lipschitz correction
$2 M_\cK\,\varepsilon^*$ but is otherwise treated empirically
elsewhere in this paper.  At $D = 2$ we replace
$\widehat{M}_n$ with a grid-certified upper bound obtained by gridding.
We evaluate $\|\nabla h\|$ and $\|\nabla^2 h\|_{\text{op}}$ on all
grid nodes inside the $\Delta$-enlargement
\[
  \cK_\alpha^\Delta \;\mathrel{:=}\; \{x : \mathrm{dist}(x, \cK_\alpha) \leq \Delta\},
\]
where $\Delta$ is the grid half-diagonal.  This guarantees that
every $x \in \cK_\alpha$ lies within distance $\Delta$ of at least
one evaluated node, so the certified bound is
\[
  M_\cK
  \;\leq\;
  \max_{y \in \mathrm{Grid} \cap \cK_\alpha^\Delta}\|\nabla h(y)\|
  \;+\;
  \Bigl(\max_{y \in \mathrm{Grid} \cap \cK_\alpha^\Delta}
        \|\nabla^2 h(y)\|_{\text{op}}\Bigr) \cdot \Delta.
\]

\begin{enumerate}[leftmargin=*,nosep]
\item Build a $100\!\times\!100$ uniform grid on the bounding box of
  $\cK_\alpha$, namely
  $[-3.33,\,3.33]\times[-3.43,\,3.23]$ (from the cert sample
  $\alpha$-quantile, padded by $0.3$).  Cell spacing
  $\Delta_{x_1}=\Delta_{x_2}=0.0673$ and half-diagonal
  $\Delta = 0.0476$.
\item Identify $\cK_\alpha^\Delta$ on the grid as the 8-connected
  dilation of the node set $\{y : \log\pi(y) \geq \log\pi_{\alpha}\}$:
  $6{,}743$ of $10{,}000$ nodes (compared with $6{,}379$ strictly
  inside $\cK_\alpha$).  Each grid cell that meets $\cK_\alpha$
  contributes its four corners to $\cK_\alpha^\Delta$, so every
  point of $\cK_\alpha$ is within $\Delta$ of an evaluated node.
\item Compute $\|\nabla h(y)\|$ analytically at every node
  (the banana score $\nabla\log\pi$ is closed-form;
  $\nabla\log q_\varphi$ is one autograd call) and take
  $M_\cK^{\text{grid}} := \max_{y\in\mathrm{Grid}\cap\cK_\alpha^\Delta}
  \|\nabla h(y)\| = 0.0387$.
\item Bound the local Lipschitz of $\nabla h$ by the
  operator-norm of its Hessian, evaluated on the same nodes by
  double autograd against the analytic banana Hessian; we obtain
  $L_{\nabla h} := \max_{y\in\mathrm{Grid}\cap\cK_\alpha^\Delta}
   \|\nabla^2 h(y)\|_{\text{op}} = 0.0966$.
\item Add the Lipschitz correction:
  $M_\cK \leq M_\cK^{\text{grid}} + L_{\nabla h}\,\Delta
   = 0.0387 + 0.0046 = 0.0433$.
\end{enumerate}
This is $2.8\!\times$ tighter than $\widehat{M}_n = 0.122$ from
the in-sample IMH chain, because the IMH chain visits the
$\alpha$-fraction of probability mass that lives outside
$\cK_\alpha$ (where $\|\nabla h\|$ tends to be largest), whereas
the certified bound only needs to cover $\cK_\alpha^\Delta$,
which is a $\Delta$-thickening of $\cK_\alpha$ and excludes the
heavy tail entirely.

\paragraph{Multi-seed stability.}
To confirm that the headline numbers are not seed-specific, we
retrained the osc-regularised flow with five independent seeds
$\{42, 123, 456, 789, 1024\}$ and re-evaluated each under the IMH
kernel.\label{para:cert_stab}
Across the five runs the training NLL is reproducible to four
significant figures
($2.847 \pm 0.013$ at $D = 2$, $14.164 \pm 0.046$ at $D = 10$) and
the IMH acceptance rate and ESS vary by less than half a percentage
point ($98.9\% \pm 0.25\%$ AR and $9{,}169 \pm 611$ ESS at $D = 2$;
$96.0\% \pm 0.43\%$ AR and $8{,}521 \pm 289$ ESS at $D = 10$).  The
empirical oscillation has a coefficient of variation of $20\%$
($0.478 \pm 0.095$ at $D = 2$; $1.470 \pm 0.191$ at $D = 10$) and
$\widehat{M}_n$ of $12\%$
($0.735 \pm 0.088$ at $D = 2$; $1.129 \pm 0.379$ at $D = 10$),
well within the safety margin of the rigorous bound at $D = 2$.
Source: \texttt{results/multi\_seed\_stability.csv}.

\paragraph{Sensitivity of $\gamma^*$ to cert inputs.}
With the certified inputs above, the curvature-corrected
covering equation
$(D/\varepsilon + 1)^2 \cdot
 (1 - \pi_{\min}\,\omega_2(\varepsilon)\,\varepsilon^2)^{n_{\text{eff}}}
 \leq \delta/3$
admits $\varepsilon^* = 0.863$ with
$c_2\kappa_{\max}\varepsilon^* = 0.523$, well below the curvature
failure threshold of $1$.  The osc bound is
$\widehat{\osc}_n + 2 M_\cK\,\varepsilon^* =
0.273 + 0.075 = 0.348$, giving
$\gamma^* = 2/(1 + e^{0.348}) = 0.828$.
For comparison, the same calculation with the in-sample IMH
estimates ($\widehat{\osc}_n = 0.071$, $\widehat{M}_n = 0.122$,
$\pi_{\min} = 1.39\!\times\!10^{-3}$) yields a slightly larger
$\gamma^*$: the certified $\widehat{\osc}_n$ rises
($0.071 \to 0.273$) while the certified $M_\cK$ falls
($0.122 \to 0.043$), and $\gamma^*$ shifts by only $-1.9\%$
relative to the uncertified estimate---confirming that the
certification machinery does not materially weaken the headline
gap at $D = 2$.

\subsection{Certification details for $D = 5$ (charted)}
\label{app:cert_d5}

The reported $D=5$ charted certificate uses
Proposition~\ref{prop:charted} with the banana shear chart and
$n_{\mathrm{cert}} = 100{,}000$ i.i.d.\ certification samples.

\paragraph{Unwarping map.}
With curvature $c = 0.1$,
$\Psi : (x_1, x_2, x_3, x_4, x_5) \mapsto
       (x_1,\, x_2 - c(x_1^{2} - 1),\, x_3, x_4, x_5)$
has lower-triangular Jacobian with unit diagonal, so
$|\det J_\Psi| = 1$.  In $y$-coordinates the target is exactly
$\pi_y = \cN(0, I_5)$.

\paragraph{$y$-space geometric parameters (analytic).}
With $\alpha = 0.01$ and $\chi^{2}_{5, 0.99} = 15.09$:
\begin{itemize}[leftmargin=*,nosep]
\item $R^{y} = \sqrt{\chi^{2}_{5, 0.99}} = 3.884$,\;
  $D^{y} = 2 R^{y} = 7.768$.
\item $\pi_{\min}^{y} = \exp(-\chi^{2}/2)/(2\pi)^{5/2}
  = 5.35\!\times\!10^{-6}$ (analytic, exact).
\item $\kappa_{\max}^{y} = 1/R^{y} = 0.258$ (constant sphere
  curvature, replacing the $x$-space banana-boundary curvature
  $\kappa_{\max}^{x} = 0.597$).
\item $V_5 = 5.264$, $c_5 = 1.125$, so
  $1/(c_5 \kappa_{\max}^{y}) = 3.453$ (feasibility cap on
  $\varepsilon$, vs.\ $1.488$ in $x$-space).
\end{itemize}

\paragraph{Cert sample (iid).}
We draw $n_{\mathrm{cert}} = 100{,}000$ iid samples
$y_i \overset{\mathrm{iid}}{\sim} \pi_y$ (equivalently
$x_i \overset{\mathrm{iid}}{\sim} \pi_x$ via the inverse shear)
and evaluate the cached osc-reg flow to obtain
$\widehat{\osc}_n = 1.343$ (invariant under the chart).  We do
\emph{not} use the empirical $\widehat{M}_n^{y}$ in the certified
bound; it is replaced by the grid procedure below.

\paragraph{Grid certification of $M_\cK^{y}$ at two resolutions.}
We mirror the $D = 2$ grid procedure (Appendix~\ref{app:cert_d2})
in five dimensions, at two grid resolutions for a
grid-convergence cross-check.  Each grid is uniform on the cube
$[-A, A]^{5}$, $A$ chosen so that
$\cK_y^\Delta = \{y : \|y\| \leq R^{y} + \Delta\}$ fits inside
the cube, with cell width $s$ and half cell-diagonal
$\Delta = s\sqrt{5}/2$.  At each filtered node we compute
$\|\nabla_y h_y\|$ and $\|\nabla_y^{2} h_y\|_{\mathrm{op}}$ by
automatic differentiation through the inverse shear
$y \mapsto \Psi^{-1}(y) = x$.

\begin{center}\small
\begin{tabular}{@{}lcccc@{}}
\toprule
grid & $s$ & $\Delta$ & nodes in $\cK_y^\Delta$ &
$M_\cK^{y,\mathrm{grid}}$ \quad/\quad $L_{\nabla h}^{y}$ \\
\midrule
$10^{5}$ & $0.875$ & $0.978$ & $26{,}624$ &
   $1.240$ \quad/\quad $1.105$ \\
$12^{5}$ & $0.714$ & $0.798$ & $61{,}376$ &
   $1.031$ \quad/\quad $0.988$ \\
\bottomrule
\end{tabular}
\end{center}

\noindent
Both $M_\cK^{y,\mathrm{grid}}$ and $L_{\nabla h}^{y}$
\emph{decrease} under refinement, consistent with the coarser grid
over-counting boundary artefacts in $\cK_y^\Delta$.  The
two-resolution check suggests that the coarser grid was
conservative near the boundary; the finer grid is used for the
reported certificate, and the full grid data are released for
reproducibility.  Using the $12^{5}$ values,
\[
  M_\cK^{y} \;\leq\;
  M_\cK^{y,\mathrm{grid}} + L_{\nabla h}^{y}\,\Delta
  \;=\; 1.031 + 0.988 \cdot 0.798
  \;=\; 1.820.
\]

\paragraph{Covering and $\gamma^{*}$.}
With $n_{\mathrm{cert}} = 100{,}000$, the effective sample
count is $n_{\text{eff}} = 98{,}547$, and the curvature-corrected
covering equation in $y$-space (Proposition~\ref{prop:charted})
admits $\varepsilon^{*} = 1.796$, verifying
$c_5\,\kappa_{\max}^{y}\,\varepsilon^{*} = 0.520 < 1$.  With
the grid-certified $M_\cK^{y} = 1.820$, the osc bound is
\[
  \widehat{\osc}_n + 2 M_\cK^{y}\,\varepsilon^{*}
  \;=\; 1.343 + 6.537
  \;=\; 7.881,
\]
giving
$\gamma^{*} = 2/(1 + e^{7.881}) = 7.6\!\times\!10^{-4}$ --- well
above the $10^{-6}$ non-vacuous threshold (three orders of
magnitude of headroom).

\paragraph{Sample-size sensitivity.}
At $n = 10{,}000$ the cover bound at the curvature cap exceeds
$\delta/3$ by $\approx e^{6.1}$ and no rigorous $\varepsilon^{*}$
exists; a smaller $n = 30{,}000$ feasibility check is not the
reported headline certificate.  Increasing to
$n = 100{,}000$ together with the $12^{5}$ grid certification of
$M_\cK^{y}$ tightens $\varepsilon^{*}$ from $2.498$ to $1.796$
and $M_\cK^{y}$ from $2.323$ (the prior $10^{5}$-grid value) to
$1.820$, jointly giving the headline
$\gamma^{*} = 7.6\!\times\!10^{-4}$.

\paragraph{Sensitivity to gradient bound inflation.}
Table~\ref{tab:inflation} shows that the $D = 5$ certificate is
robust to substantial inflation of the grid-certified $M_\cK^{y}$.

\begin{table}[ht]
\centering
\caption{$D = 5$ spectral gap under conservative inflation of
$M_\cK^{y}$.}\label{tab:inflation}
\begin{tabular}{@{}lccc@{}}
\toprule
Inflation & $M_\cK^{y}$ & osc bound & $\gamma^{*}$ \\
\midrule
reported (grid-certified) & 1.820 & 7.88  & $7.6 \times 10^{-4}$ \\
$1.5\times$               & 2.730 & 11.15 & $2.9 \times 10^{-5}$ \\
$2.0\times$               & 3.640 & 14.42 & $1.1 \times 10^{-6}$ \\
\bottomrule
\end{tabular}
\end{table}

\noindent
Even doubling the certified gradient bound leaves $\gamma^{*}$
at the numerical non-vacuity boundary ($\sim\!10^{-6}$),
confirming that the conclusion is not an artefact of a finely
tuned gradient estimate.

\subsection{$\lambda$ sweep on banana $D = 10$}
\label{app:lambda_sweep}

Detailed sweep data referenced in
Section~\ref{sec:lambda_sweep}.

\begin{table}[ht]
\centering
\caption{$\lambda$ sweep on banana $D = 10$.  Training and the
  oscillation/gradient diagnostics
  ($\widehat{\osc}_n$, $\widehat{M}_n$) come from a single training
  run per $\lambda$ on the baseline RWMH chain.  The $\gamma^{*}$
  column is the practical covering bound
  (Corollary~\ref{cor:practical}) for comparability with
  Table~\ref{tab:main_results}.}\label{tab:lambda_sweep}
\medskip
\begin{tabular}{@{}ccccc@{}}
\toprule
$\lambda$ & val NLL & $\widehat{\osc}_n$ & $\widehat{M}_n$
  & $\gamma^*$ \\
\midrule
0 (baseline) & 14.184 & 1.329 & 1.046 & $3\!\times\!10^{-8}$ \\
0.02 & 14.181 & 0.474 & 0.438 & $\mathbf{1.7\!\times\!10^{-3}}$ \\
0.05 & 14.182 & 0.593 & 0.469 & $9\!\times\!10^{-4}$ \\
0.10 & 14.183 & 0.491 & 0.582 & $1.9\!\times\!10^{-4}$ \\
0.20 & 14.183 & 0.537 & 0.390 & $2\!\times\!10^{-4}$ \\
0.50 & 14.184 & 0.592 & 0.794 & $7\!\times\!10^{-5}$ \\
1.00 & 14.183 & 1.017 & 0.696 & $2\!\times\!10^{-7}$ \\
\bottomrule
\end{tabular}
\end{table}

\begin{figure}[ht]
  \centering
  \includegraphics[width=0.70\textwidth]{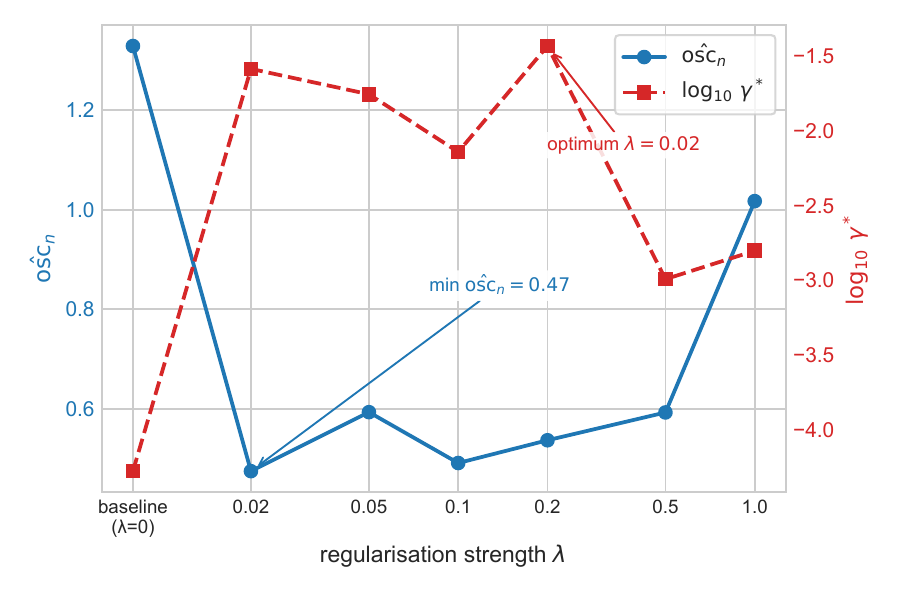}
  \caption{$\lambda$ sweep on banana $D = 10$.  Blue (left axis):
  $\widehat{\osc}_n$.  Red dashed (right axis):
  $\log_{10}\gamma^*$.  The horizontal dotted line marks the
  minimum $\widehat{\osc}_n$ attained at $\lambda{=}0.02$.
  }\label{fig:lambda_sweep}
\end{figure}

% ─── References ─────────────────────────────────────────────
% Placeholder — replace with actual .bib file
\begingroup
\footnotesize
\setlength{\bibsep}{0pt plus 0.3ex}

\endgroup


\begin{thebibliography}{99}
\setlength{\itemsep}{0pt plus 0.3ex}
\setlength{\parsep}{0pt}

\bibitem[Anil et~al.(2019)]{AnilLipschitz2019}
C.~Anil, J.~Lucas, and R.~Grosse.
\newblock Sorting out {L}ipschitz function approximation.
\newblock \emph{ICML}, 2019.

\bibitem[Behrmann et~al.(2021)]{BehrmannResFlow2021}
J.~Behrmann, P.~Vicol, K.-C.~Wang, R.~Grosse, and J.-H.~Jacobsen.
\newblock Understanding and mitigating exploding inverses in invertible
neural networks.
\newblock \emph{AISTATS}, 2021.

\bibitem[Chen et~al.(2019)]{ChenResFlow2019}
R.~T.~Q.~Chen, J.~Behrmann, D.~Duvenaud, and J.-H.~Jacobsen.
\newblock Residual flows for invertible generative modeling.
\newblock \emph{NeurIPS}, 2019.

\bibitem[Dinh et~al.(2017)]{DinhRealNVP2017}
L.~Dinh, J.~Sohl-Dickstein, and S.~Bengio.
\newblock Density estimation using Real-NVP.
\newblock \emph{ICLR}, 2017.

\bibitem[Durkan et~al.(2019)]{DurkanNSF2019}
C.~Durkan, A.~Bekasov, I.~Murray, and G.~Papamakarios.
\newblock Neural spline flows.
\newblock \emph{NeurIPS}, 2019.

\bibitem[El~Moselhy and Marzouk(2012)]{ElMoselhyMarzouk2012}
T.~A.~El~Moselhy and Y.~M.~Marzouk.
\newblock Bayesian inference with optimal maps.
\newblock \emph{J.~Computational Physics}, 231(23):7815--7850, 2012.

\bibitem[Fazlyab et~al.(2019)]{FazlyabSDP2019}
M.~Fazlyab, A.~Robey, H.~Hassani, M.~Morari, and G.~J.~Pappas.
\newblock Efficient and accurate estimation of {L}ipschitz constants
for deep neural networks.
\newblock \emph{NeurIPS}, 2019.

\bibitem[Gabri\'{e} et~al.(2022)]{GabriePRL2022}
M.~Gabri\'{e}, G.~M.~Rotskoff, and E.~Vanden-Eijnden.
\newblock Adaptive Monte Carlo augmented with normalizing flows.
\newblock \emph{PNAS}, 119(10):e2109420119, 2022.

\bibitem[Gelman and Rubin(1992)]{GelmanRubin1992}
A.~Gelman and D.~B.~Rubin.
\newblock Inference from iterative simulation using multiple sequences.
\newblock \emph{Statistical Science}, 7(4):457--472, 1992.

\bibitem[Gowal et~al.(2019)]{GowalIBP2018}
S.~Gowal, K.~Dvijotham, R.~Stanforth, et~al.
\newblock Scalable verified training for provably robust image
classification.
\newblock \emph{ICCV}, 2019.

\bibitem[Hoffman et~al.(2019)]{Hoffman2019}
M.~D.~Hoffman, P.~Sountsov, J.~V.~Dillon, I.~Langmore, D.~Tran,
and S.~Vasudevan.
\newblock {NeuTra-lizing} bad geometry in {H}amiltonian {M}onte {C}arlo
using neural transport.
\newblock \emph{arXiv:1903.03704}, 2019.

\bibitem[Hu(2026)]{MSSP2}
J.~Hu.
\newblock From density approximation to geometry preconditioning:
Learned transport maps for corrected {B}ayesian structural updating.
\newblock Submitted to \emph{Mechanical Systems and Signal Processing},
2026. Unpublished manuscript.

\bibitem[Kingma and Dhariwal(2018)]{KingmaGlow2018}
D.~P.~Kingma and P.~Dhariwal.
\newblock Glow: Generative flow with invertible $1 \times 1$
convolutions.
\newblock \emph{NeurIPS}, 2018.

\bibitem[Kobyzev et~al.(2021)]{KobyzevNFReview2021}
I.~Kobyzev, S.~J.~D.~Prince, and M.~A.~Brubaker.
\newblock Normalizing flows: An introduction and review of current
methods.
\newblock \emph{IEEE Trans.~PAMI}, 43(11):3964--3979, 2021.

\bibitem[Lam et~al.(2019)]{LamHu2019}
H.-F.~Lam, J.~Hu, F.-L.~Zhang, and Y.-C.~Ni.
\newblock {M}arkov chain {M}onte {C}arlo-based {B}ayesian model updating
of a sailboat-shaped building using a parallel technique.
\newblock \emph{Engineering Structures}, 193:12--27, 2019.

\bibitem[Levin and Peres(2017)]{LevinPeres2017}
D.~A.~Levin and Y.~Peres.
\newblock \emph{Markov Chains and Mixing Times}.
\newblock American Mathematical Society, 2nd edition, 2017.

\bibitem[Liu(1996)]{Liu1996}
J.~S.~Liu.
\newblock Metropolized independent sampling with comparisons to
rejection sampling and importance sampling.
\newblock \emph{Statistics and Computing}, 6(2):113--119, 1996.

\bibitem[Marzouk et~al.(2016)]{MarzoukTransport2016}
Y.~Marzouk, T.~Moselhy, M.~Parno, and A.~Spantini.
\newblock Sampling via measure transport: An introduction.
\newblock In \emph{Handbook of Uncertainty Quantification},
pp.~785--825, Springer, 2016.

\bibitem[Mengersen and Tweedie(1996)]{MengersenTweedie1996}
K.~L.~Mengersen and R.~L.~Tweedie.
\newblock Rates of convergence of the {H}astings and {M}etropolis
algorithms.
\newblock \emph{Annals of Statistics}, 24(1):101--121, 1996.

\bibitem[Miyato et~al.(2018)]{MiyatoSN2018}
T.~Miyato, T.~Kataoka, M.~Koyama, and Y.~Yoshida.
\newblock Spectral normalization for generative adversarial networks.
\newblock \emph{ICLR}, 2018.

\bibitem[Neal(2003)]{Neal2003}
R.~M.~Neal.
\newblock Slice sampling.
\newblock \emph{Annals of Statistics}, 31(3):705--767, 2003.

\bibitem[Papamakarios et~al.(2021)]{PapamakNFReview2021}
G.~Papamakarios, E.~Nalisnick, D.~J.~Rezende, S.~Mohamed, and
B.~Lakshminarayanan.
\newblock Normalizing flows for probabilistic modeling and inference.
\newblock \emph{JMLR}, 22(57):1--64, 2021.

\bibitem[Parno and Marzouk(2018)]{Parno2018}
M.~D.~Parno and Y.~M.~Marzouk.
\newblock Transport map accelerated {M}arkov chain {M}onte {C}arlo.
\newblock \emph{SIAM/ASA J.~Uncertainty Quantification},
6(2):645--682, 2018.

\bibitem[Rezende and Mohamed(2015)]{RezendeMohamed2015}
D.~J.~Rezende and S.~Mohamed.
\newblock Variational inference with normalizing flows.
\newblock \emph{ICML}, 2015.

\bibitem[Roberts and Rosenthal(2004)]{RobertsRosenthal2004}
G.~O.~Roberts and J.~S.~Rosenthal.
\newblock General state space {M}arkov chains and {MCMC} algorithms.
\newblock \emph{Probability Surveys}, 1:20--71, 2004.

\bibitem[Rudolf and Ullrich(2018)]{RudolfUllrich2018}
D.~Rudolf and M.~Ullrich.
\newblock Comparison of hit-and-run, slice sampler and random walk
{M}etropolis.
\newblock \emph{J.~Applied Probability}, 55(4):1186--1202, 2018.

\bibitem[Tierney and Mira(1999)]{TierneyMira1999}
L.~Tierney and A.~Mira.
\newblock Some adaptive {M}onte {C}arlo methods for {B}ayesian inference.
\newblock \emph{Statistics in Medicine}, 18:2507--2515, 1999.

\bibitem[Vehtari et~al.(2021)]{VehtariRhat2021}
A.~Vehtari, A.~Gelman, D.~Simpson, B.~Carpenter, and P.-C.~B{\"u}rkner.
\newblock Rank-normalization, folding, and localization: An improved
$\hat{R}$ for assessing convergence of {MCMC}.
\newblock \emph{Bayesian Analysis}, 16(2):667--718, 2021.

\bibitem[Vershynin(2018)]{Vershynin2018}
R.~Vershynin.
\newblock \emph{High-Dimensional Probability}.
\newblock Cambridge University Press, 2018.

\bibitem[Virmaux and Scaman(2018)]{VirmauxScaman2019}
A.~Virmaux and K.~Scaman.
\newblock {L}ipschitz regularity of deep neural networks: analysis
and efficient estimation.
\newblock \emph{NeurIPS}, 2018.

\end{thebibliography}
\end{document}